\documentclass{article}

\usepackage{microtype}
\usepackage{graphicx}
\usepackage{subcaption}
\usepackage{booktabs} % for professional tables
\usepackage{tabularx}
\usepackage{hyperref}

\usepackage[preprint]{icml2026}

\usepackage{amsmath}
\usepackage{amssymb}
\usepackage{mathtools}
\usepackage{amsthm}

\usepackage[capitalize,noabbrev]{cleveref}

\theoremstyle{plain}
\newtheorem{theorem}{Theorem}[section]

\newtheorem{lemma}[theorem]{Lemma}

\theoremstyle{definition}

\theoremstyle{remark}

\usepackage[textsize=tiny]{todonotes}

\usepackage[most]{tcolorbox}
\usepackage{colortbl}
\usepackage{fontawesome5}
\usepackage{amsfonts}  % 提供 \mathbb
\usepackage{amsmath}   % 提供 \DeclareMathOperator
\usepackage{amssymb}   % 提供额外的数学符号
\usepackage{amsthm}  % 提供定理环境

\icmltitlerunning{DR-MMSearchAgent: Deepening Reasoning in Multimodal Search Agents}

\definecolor{promptheader}{RGB}{50,50,50} % 深灰色标题背景
\definecolor{promptbody}{gray}{0.95}      % 浅灰色内容背景
\usepackage{enumitem}
\usepackage[export]{adjustbox}
\usepackage{dsfont}
\usepackage{bbm}

\usepackage{tabularx}
\begin{document}

\twocolumn[
  \icmltitle{DR-MMSearchAgent: Deepening Reasoning in Multimodal Search Agents}
%ExplorationMMSearchAgent: Deep and Reliable Multimodal Search Agent


  \icmlsetsymbol{equal}{*}

  \begin{icmlauthorlist}
    \icmlauthor{Shengqin Wang}{aa}
    \icmlauthor{Wentao Yan}{aa}
    \icmlauthor{Huichi Zhou}{cc,ee}
    \icmlauthor{Yihang Chen}{cc,ee}
    \icmlauthor{Kun Shao}{dd}
    \icmlauthor{Zhizhong Zhang}{aa}
    \icmlauthor{Yuan Xie}{aa,bb}

  \end{icmlauthorlist}

  \icmlaffiliation{aa}{East China Normal University}
  \icmlaffiliation{bb}{Shanghai Innovation Institute}  
  \icmlaffiliation{cc}{University College London}  
  \icmlaffiliation{dd}{Independent Researcher}  
  \icmlaffiliation{ee}{Huawei Noah’s Ark Lab} 
  % \icmlaffiliation{}{Shanghai Innovation Institute}
  % \icmlaffiliation{}{ Independent Researcher}

  \icmlcorrespondingauthor{Yuan Xie}{yxie@cs.ecnu.edu.cn}
  %\icmlcorrespondingauthor{Firstname2 Lastname2}{first2.last2@www.uk}

  \icmlkeywords{Machine Learning, ICML}

  \vskip 0.3in
]
\printAffiliationsAndNotice{}  % no special notice (required even if empty)
% \maketitle
%图1。全文检查，补充材料（错误分析，工具泛化分析TIRbench,提示词），方法部分：加符号是扩大负优势信号进一步探索。理想状态是全部的多样轨迹，这时候得分全部一样

\begin{abstract}

Agentic multimodal models have garnered significant attention for their ability to leverage external tools to tackle complex tasks. However, it is observed that such agents often meet premature interaction collapse, caused by two primary reasons: 1) the terminal reward often appending on the last token prevents the advantage from distinguishing trajectories with exploratory behavior; 2) excessively redundant context hinders the agent from absorbing useful feedback. To address these issues, we propose the Deepening Reasoning MMSearchAgent, the framework leverages the structural proximity to derive advantage signals from the whole rollout trajectories in an entire batch, such that trajectories of different lengths are further encouraged to be generated, even when containing the same correct answer. Additionally, differentiated gaussian rewards are employed to dynamically calibrate interaction tolerance, thereby ensuring information reliability and reduce redundancy. To support multi-turn interaction training, we have
constructed a multi-step deep-reasoning dataset including 3602 high-quality QA pair with at least 3 reasonning steps. Extensive experiments demonstrate that our method achieves state-of-the-art performance, outperforming the MMSearch-R1 by 8.4$\%$ on FVQA-test.

\end{abstract}
% The framework quantifies structural proximity through rewards and response lengths to enable fine-grained advantage differentiation, enhancing signal diversity by integrating positional context.

% The framework leverages the structural proximity to derive advantage signals from the whole roll-out trajectories in an entire batch, such that trajectories of different lengths are further encouraged to be generated, even when containing the same correct answer.    
\section{Introduction}
\label{sec:intro}
% \begin{figure}[htbp]
%     \centering  % 【关键修改】让整个图组在页面居中
%     \begin{subfigure}[b]{0.19\textwidth}
%         \centering
%         \includegraphics[width=\linewidth]{fig1a.png}
%         \label{fig:a}
%         \vspace{-10pt}
%         \caption{}
%     \end{subfigure}
%     \begin{subfigure}[b]{0.20\textwidth}
%         \centering
%         \includegraphics[width=\linewidth]{fig1bv1.png}
%         \label{fig:b}
%         \vspace{-10pt}
%         \caption{}
%     \end{subfigure}
%     \caption{(a) Radar chart of comparisons on accuracy. (b) Comparison of trajectory diversity and redundancy within the same reward (non-zero reward). Diverse Trajectories refer to the statistical diversity of distinct response lengths given the same reward. Conversely, Redundant Trajectories represent the statistics of all instances featuring repeated rewards and identical response lengths.}
%     \label{fig:1}
% \end{figure}

\begin{figure}[ht]
    \centering  % 【关键修改】让整个图组在页面居中
    \begin{subfigure}[b]{0.5\linewidth}%{0.19\textwidth}
        \centering
        \includegraphics[width=\linewidth]{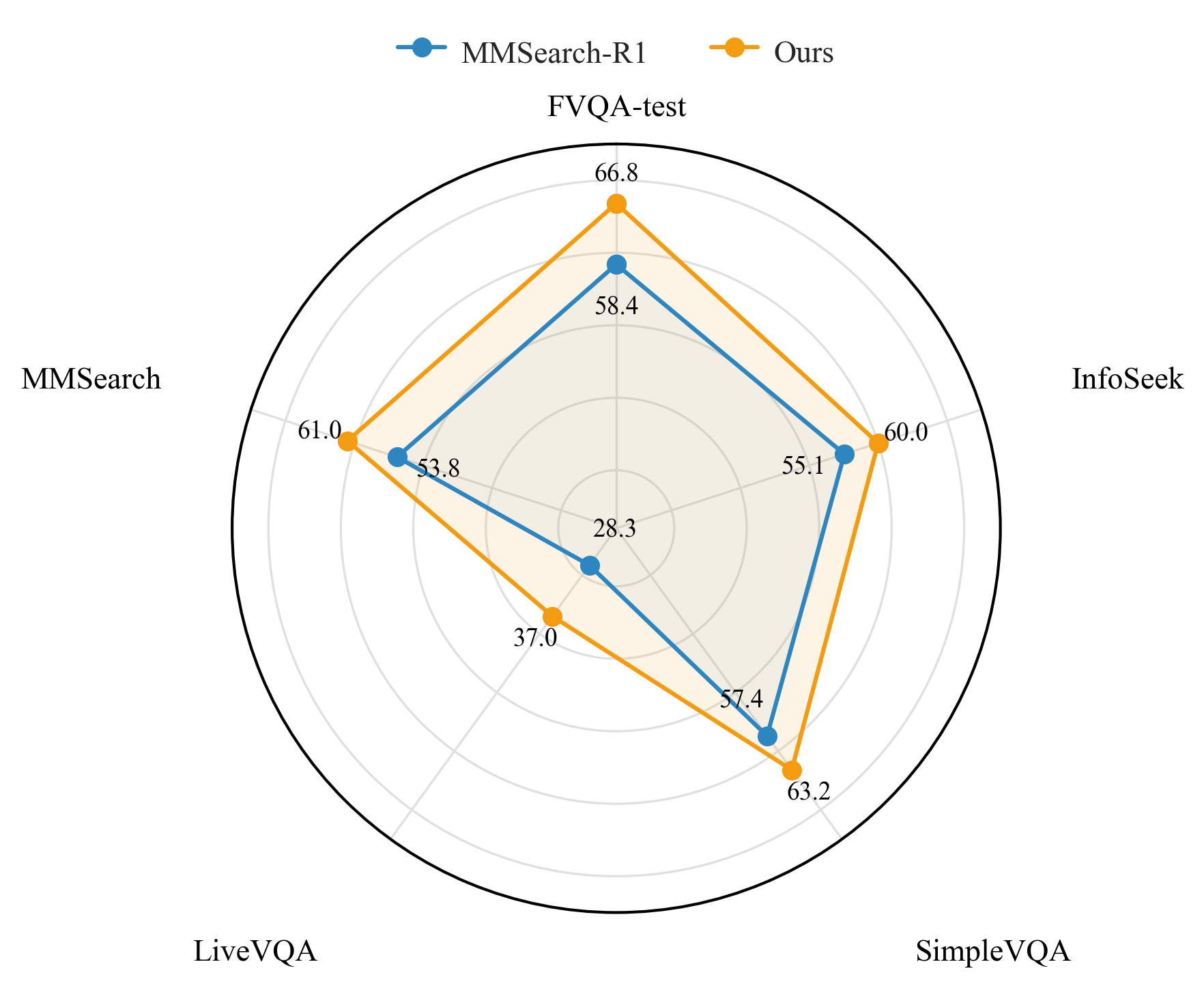}
        \label{fig:a}
        \vspace{-12pt}
        \caption{}
    \end{subfigure}%
    \begin{subfigure}[b]{0.5\linewidth}%{0.20\textwidth}
        \centering
        \includegraphics[width=\linewidth]{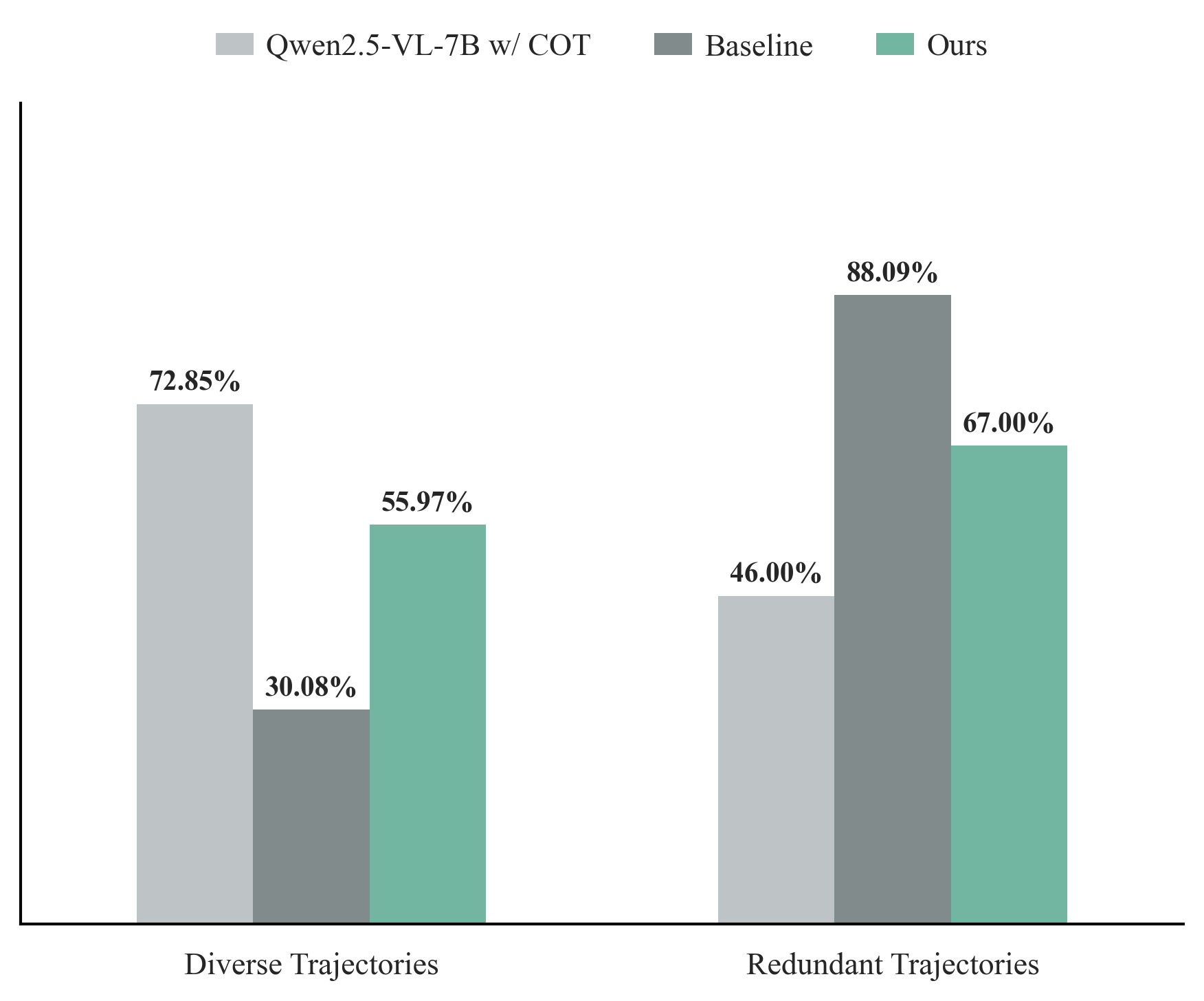}
        \label{fig:b}
        \vspace{-12pt}
        \caption{}
    \end{subfigure}
    \par%\vspace{1pt}
    
    % --- 图 (c) 的部分 ---
    % 宽度您可以根据需要调整，比如 0.5\textwidth 或 0.6\textwidth
    \begin{subfigure}[b]{0.9\linewidth}%{0.32\textwidth}
        % --- 关键点 2：子图内部居中 ---
        % 这句命令让图片在这个 0.5\textwidth 宽的盒子里居中显示
        % 如果少了这句，图片就会贴着盒子的左边，导致看起来偏左
        \centering
    
        % 插入图片，宽度设为容器宽度的 100% (\linewidth)
        % 如果图c原始比例很宽，也可以手动指定宽度，如 width=8cm
        \includegraphics[width=1.05\linewidth]{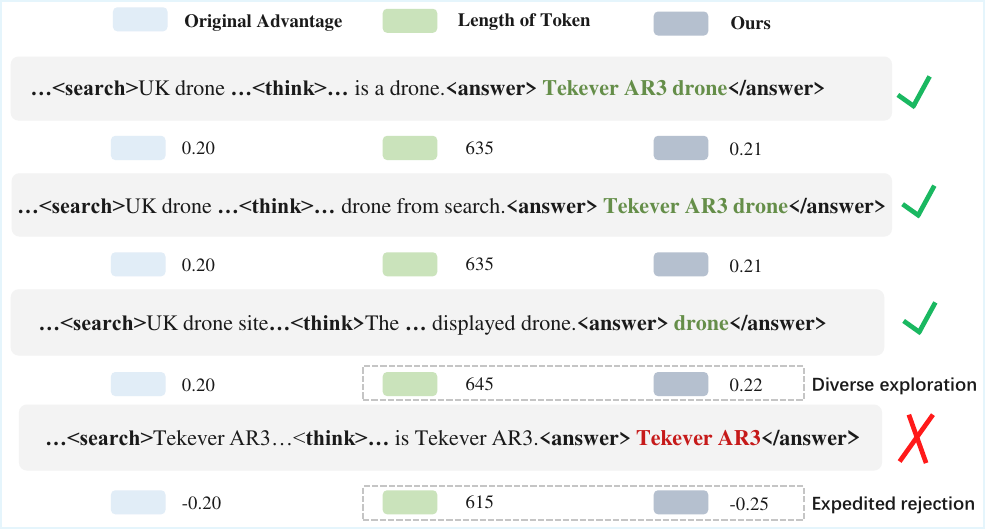}
        \caption{}
        \label{fig:c}
    \end{subfigure}
    
    \caption{(a) Radar chart of comparisons on accuracy. (b) Comparison of trajectory diversity within the same reward and redundancy (non-zero reward). Diverse Trajectories denote the percentage of trajectories with varying response lengths under identical rewards, while Redundant Trajectories represent the percentage of trajectories with duplicated rewards and identical response lengths. (c) Comparison of our advantage calculation method with current methods.}
    \label{fig:1}
\end{figure}

Agentic multimodal models, when addressing time-sensitive Visual Question Answering (VQA), demonstrate the capability of leveraging the  search tools to acquire comprehensive knowledge. Previous approaches have largely relied on retrieval-augmented generation (RAG) methods~\cite{chen2022murag,chen2024mllm,yu2024rankrag} and early-stage search agents~\cite{hu2023avis,jiang2024mmsearch,zhang2024vision}. While these methods excel at integrating static external knowledge, they struggle to capture complete time-sensitive information in dynamic multimodal environments, leading to ineffective searches. To bridge this gap, recent studies~\cite{wu2025mmsearch, narayan2025deepmmsearch} have successfully integrated external search tools and employed reinforcement learning to train agents capable of multi-turn reasoning. 
However, existing RL-based multimodal search methods suffer from exploration degradation. For instance, MMSearch-R1~\cite{wu2025mmsearch} tend to converge toward suboptimal two-turn tool interactions, which limits their capacity to tackle complex, multi-step tasks. This failure is primarily twofold. First, existing advantage estimation overlook the verification of reward positional information from a global perspective by comparing the whole roll-out trajectories of different lengths, resulting in coarse-grained signals that lack the diversity. Specifically, as shown in \cref{fig:1} (b), the trajectories in Chain-of-Thought (CoT) reasoning exhibit high diversity under identical rewards. Consequently, the proportion of diverse trajectories shrinks over time while repetitive ones proliferate, causing internal advantage signals to lose their discriminative power. Second, trajectory redundancy and contextual interference present further for exploration. As depicted in the left panel of \cref{fig:4}, models frequently handle long-sequence responses from external tools. Such informational erosion causes the agent to prematurely terminate reasoning or bypass essential multi-turn tool calls when facing complex objectives.

% 为了系统性地应对上述探索退化困境，我们提出了 ExplorationMMSearchAgent 框架，通过多个核心组件的协同演进，全面增强多模态工具链的深化推理。
% 首先，我们引入了结构邻近性加权优势注入机制（Structural Proximity-weighted Advantage Injection）。该机制通过考虑奖励大小以及在全局视野中的位置，量化有效的token 奖励与理想的多样推理路径在序列结构上的贴近度，将计算的轨迹得分注入原始优势，能够实现对相同奖励值下异构轨迹的有效辨识，显著拓宽了模型的多样化探索空间。其次我们设计了双向引导的自适应平滑奖励奖励，通过差异化的高斯奖励分布，指导模型回答时能够依据中间解的质量动态校准探索情况，确保交互行为的可信度。
% 此外集成了提炼智能体（Refining Agent）对冗余交互轨迹进行实时压缩与去噪。进一步地，为支撑深层推理链的训练，我们构建了专注于跨模态多跳逻辑建模的 BridgeVQA 数据集，并以此构建了一个可靠的多模态搜索服务。

% 该机制通过考虑响应长度以及奖励的全局视野，量化有效奖励token与虚拟理想路径之间的结构贴近度。将该度量注入原始优势函数，不仅可以有效地区分具有相同奖励的异构轨迹，还可以利用奖励分布的全局信息性来增加优势信号的多样性，从而鼓励模型的进一步探索。
% First, we introduce the Structural Proximity-weighted Advantage Injection (SPAI), which quantifies the structural proximity between reward trajectories by leveraging global-context information based on verifiable rewards and response lengths. This approach not only achieves fine-grained advantage differentiation among identical rewards but also incorporates the positional context of diverse rewards, thereby enhancing the diversity of advantage signals.
To address these issues, as in \cref{fig:2}, we propose the Deepening Reasoning MMSearchAgent (DR-MMSearchAgent). First, we introduce the Structural Proximity-weighted Advantage Injection (SPAI), that quantifies the structural proximity between verification reward trajectories. It therefore effectively distinguishes heterogeneous trajectories under the same reward, but with different CoT length,  enhancing the diversity of advantage signals. As shown in \cref{fig:1} (b), (c), by incorporating the positional information of verifiable rewards, this mechanism mitigates the decline in reward trajectory diversity and the redundancy of repetitive trajectories, thereby ensuring exploration of the model.

%As shown in \cref{fig:1} (b), (c), by combining response length with a global reward perspective, this mechanism quantifies the structural proximity between verification reward trajectories. By integrating the original advantage information, it effectively distinguishes heterogeneous trajectories with identical rewards, while also leveraging the length information of the reward distribution to enhance the diversity of advantage signals. 

Second, we design a Bidirectionally Guided Adaptive Smoothed (BGAS) reward. Utilizing differentiated gaussian reward distributions, this component guides the model to dynamically calibrate its exploration based on the quality of intermediate solutions, thereby ensuring the credibility of interaction behaviors. Furthermore, a refining agent is integrated to perform real-time compression and denoising of redundant interaction trajectories. Our contributions are summarized below:
\begin{figure*}[t]
    \centering
    \includegraphics[width=1.0\textwidth]{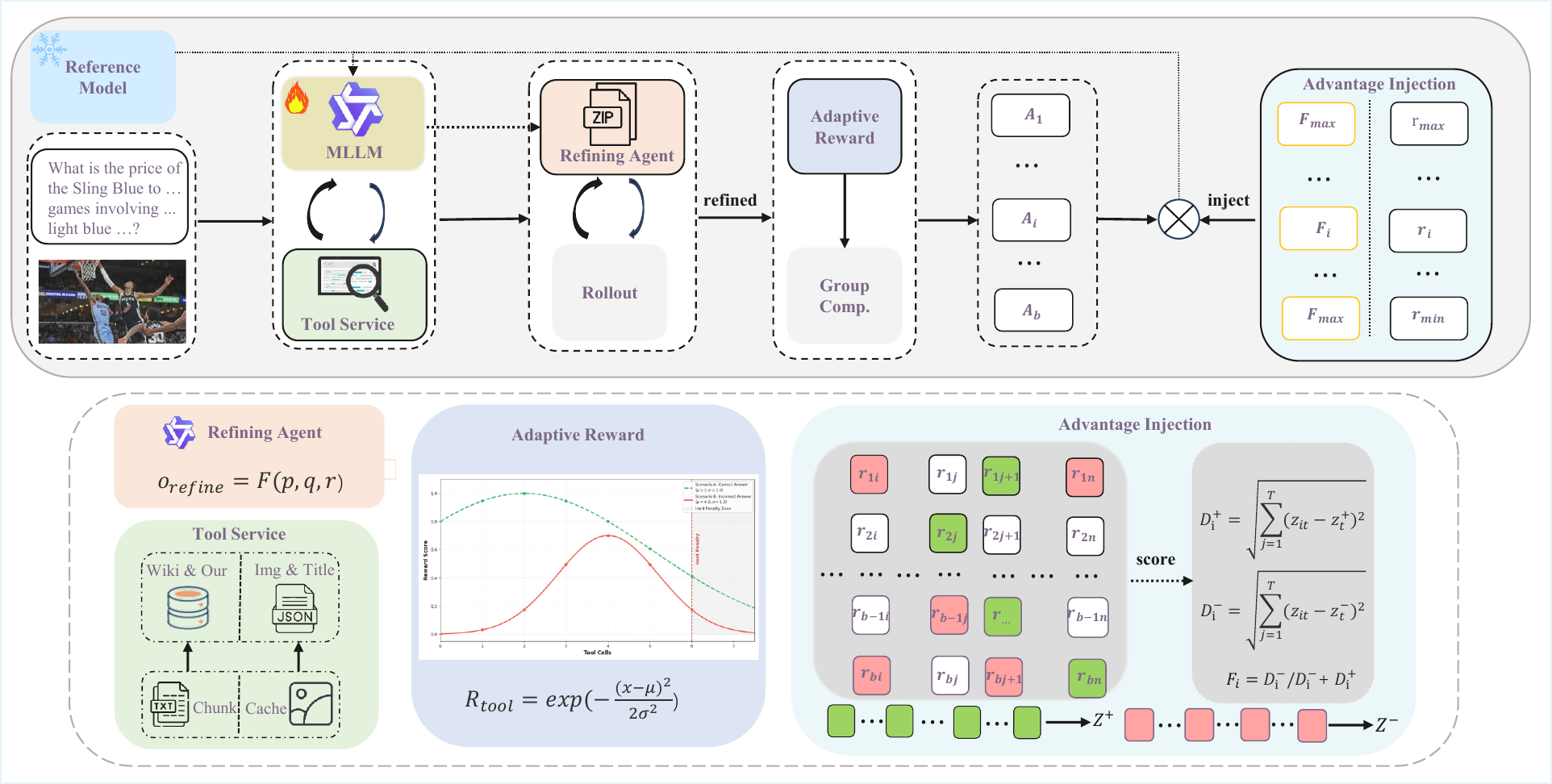} % 替换为你的图片
    \caption{Framework for DR-MMSearchAgent. The upper panel illustrates the overall training framework, while the lower panel depicts the advantage injection mechanism, the adaptive reward mechanism, the refining agent, and the updatable multimodal tool services.}
    \label{fig:2}
\end{figure*}
% \begin{figure*}[t]
%     \centering
%     \includegraphics[width=1.0\textwidth]{casev1.pdf} % 替换为你的图片
%     \caption{The comparison of multi-step interaction. Unlike traditional baseline models that exhibit short-sighted exploration and redundant trajectories, our framework enables stable and reliable interactions. It achieves further performance improvements while increasing the number of model interactions and reducing their corresponding lengths.}
%     \label{fig:3}
% \end{figure*}
\begin{figure}[tbp]
    \centering
    \includegraphics[width=0.49\textwidth]{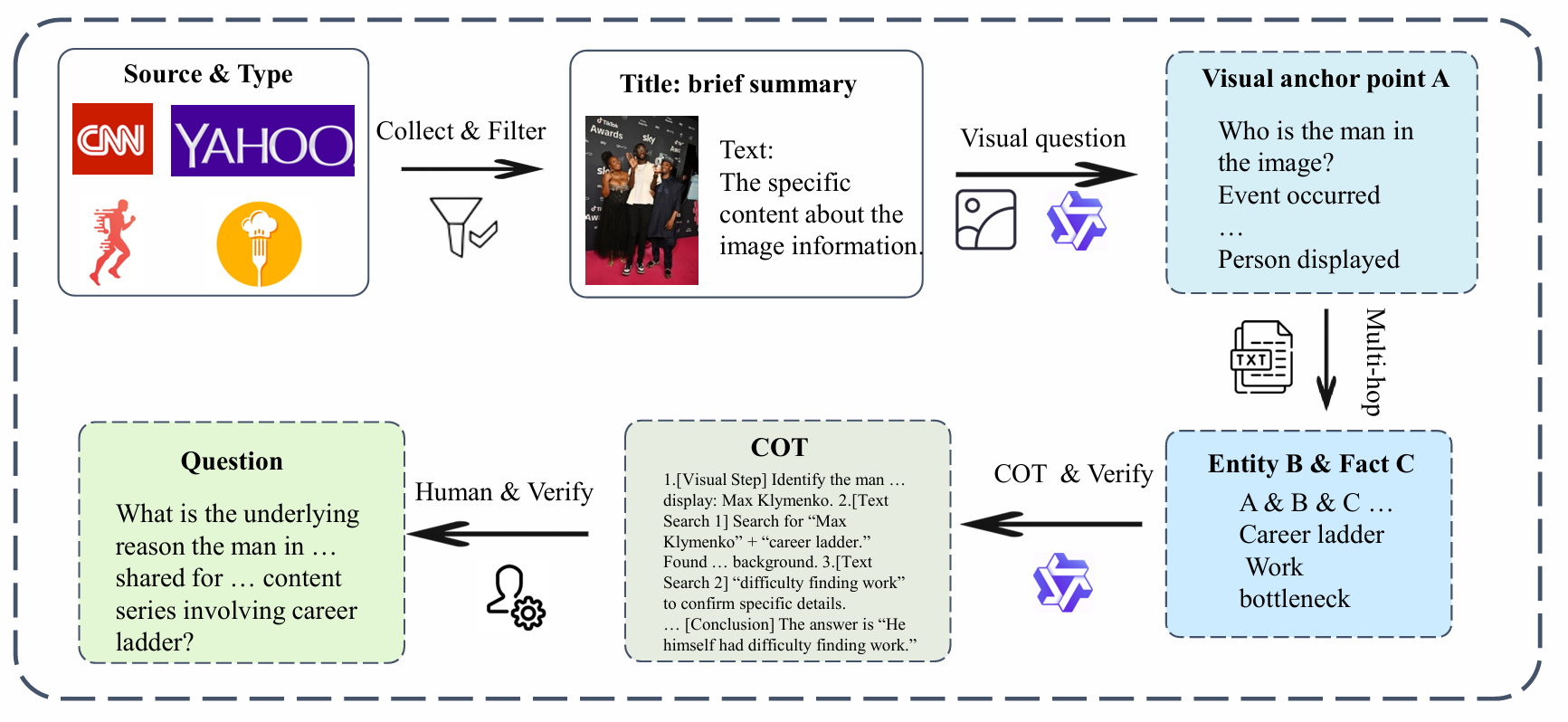} % 替换为你的图片
    \caption{Schematic of the BridgeVQA construction workflow, featuring data collection, two-step problem generation, and dual-phase filtering.}
    \label{fig:3}
\end{figure}

% \begin{figure}[ht]
%     \centering
%     \includegraphics[width=0.49\textwidth]{pipline.png} % 替换为你的图片
%     \caption{Schematic of the BridgeVQA construction workflow, featuring data collection, two-step problem generation, and dual-phase filtering.}
%     \label{fig:4}
% \end{figure}

    % \item \textbf{Characterizing the Exploration Degeneration Dilemma.} Through empirical research, we analyze the failure mechanisms of reinforcement learning in visual search tasks, identifying their roots in decision-level Exploration Myopia induced by the neglect of spatial trajectory features, and Interaction Interference caused by redundant responses at the representational level.
    
    % 该方法通过结构邻近加权机制考虑全局视野下有效token的位置信息,构造多样性优势,实现深化探索。此外，它还集成了双向引导的自适应平滑奖励机制，以动态校准模型的交互流，系统地实现了可靠的深入推理。
\textbf{Co-optimized Deep Reasoning Method.}  
The method considers the positional information of verification rewards  via a structural proximity weighting mechanism. It encourages to generate  trajectories with different CoT length to promote exploration. Furthermore, it integrates a bidirectional-guided adaptive smoothing reward mechanism to dynamically calibrate the model's interaction flow, achieving reliable deep reasoning.

\textbf{Multi-hop Reasoning MM-search Framework.}  We have constructed a multi-step deep-reasoning dataset from the latest internet visual information. Simultaneously, a refinement agent is integrated to perform real-time compression and denoising of redundant interaction trajectories.

\textbf{Performance Improvement.} Our framework sets a new state-of-the-art record on existing knowledge-intensive benchmarks while demonstrating leading capabilities on other fundamental vision tasks.
% \end{itemize}

\section{Related Work}
\label{sec:formatting}
\textbf{Multimodal Large Language Models.}
Early MLLM research achieved cross-modal alignment primarily through lightweight adapters or projection layers linking pre-trained visual encoders to LLMs~\cite{liu2024improved,chen2024internvl,liu2023visual}. Models like Qwen2.5-VL~\cite{bai2025qwen2}, InternVL3~\cite{zhu2025internvl3}, and Qwen3~\cite{yang2025qwen3} significantly improved accuracy on tasks such as VQA and image-text description generation by expanding training datasets and introducing diverse visual inputs~\cite{wu2025mmsearch,wang2025mv,zhang2024mathverse}.
% However, existing research still faces limitations: current approaches that solely rely on prompts to guide models in solving complex tasks\cite{zhang2024mathverse} exhibit significant constraints.

\textbf{Reinforcement Learning for Large Language Models.}
Reinforcement learning has become a central technique for aligning LLMs toward desired reasoning behaviors. Early methods such as PPO-based RLHF~\cite{schulman2017proximal,shao2024deepseekmath} focus primarily on preference optimization, while later frameworks~\cite{jaech2024openai,guo2025deepseek,chen2025learning,liu2023critiqueeval} aim to improve stability, sample efficiency, or reward smoothness. Recent studies~\cite{chen2022murag,chen2024mllm,yu2024rankrag} on tool-augmented LLMs explore how RL can encourage tool selection and planning.

% However, most methods assume concise reasoning trajectories and struggle when extended multi-stage interactions or tool calls become necessary. Recent studies on tool-augmented LLMs explore how RL can encourage tool selection and planning, but they often rely on sparse correctness rewards, making exploration brittle and sensitive to trajectory length.

\section{Method}
\label{sec:method}

% \subsection{GRPO in MLLMs}
\subsection{Structural Proximity-weighted Advantage Injection}
\textbf{GRPO in MLLMs:}
Group-Relative Policy Optimization 
(GRPO)~\cite{shao2024deepseekmath} stabilizes training by introducing an intra-group comparison mechanism. For a given query $q$, the policy generates a group of $N$ responses $O = \{o_1, \dots, o_N\}$, and then the reward ${R}$ across this entire group is computed. Its core objective formula is:
\begin{equation}\label{E1}
{E}_{i,t}\left[\min\left(\rho_t^{(i)} A^{(i)}, \text{clip}A^{(i)}\right)\right] - B, 
\end{equation}
where $B = \beta {KL}(\pi_{\theta} \| \pi_{\text{ref}})$, $A^{(i)} = \frac{{r}(\tau) - \text{mean}(R)}{\text{std}(R)} $.

In the advantage, ${r}(\tau)$ denotes the reward appending to the last response token, and $R$ represents the reward matrix within the whole roll-out trajectories. In its implementation, the advantage is only calculated in the terminal scalar of the verification reward, failing to account for the specific positional information of trajectory rewards.

Based on the verification reward, we construct $R' \in \mathbb{R}^{G \times T}$, where $G$ denotes the group size and $T$ represents the maximum sequence length. For each trajectory, a reward value is assigned only to the last token of the response $L_i$, with all other positions set to zero:
\begin{equation}
r_{i,t} = 
\begin{cases} 
{r}(\tau), & \text{if } t = L_i \\
0, & \text{if } 1 \leq t < L_i.
\end{cases}
\end{equation}

% 我们对token level reward进行分析，观察到明显的同分异构现象，即不同位置的奖励信息往往获得相同的终端标量奖励。这种优势信号的冗余掩盖了轨迹间的结构性差异，若不加干预，将导致探索多样性的下降和信号的同质化。
The $R'$ provide the positional information or CoT length of reward ${r}(\tau)$ by considering the whole roll-out trajectories in an entire batch. Notably, as illustrated in \cref{fig:1} (b), through $R'$, a pronounced isomorphism phenomenon appears where reward tokens at disparate positions are frequently assigned identical terminal scalar rewards. This redundancy in advantage signals obviously obscures the structural disparities between trajectories, %Furthermore, existing methods derive advantage signals solely from trajectory rewards, failing to account for the positional context of identical rewards or the global landscape of the reward structure. 
and therefore diminishes the diversity of advantage signals, ultimately resulting in insufficient exploration.
% However, current computational methods fail to account for spatial reward trajectory features, leading to the issue of exploration rigidity.

% \begin{figure}[ht]
%     \centering
%     \includegraphics[width=0.49\textwidth]{pipline.png} % 替换为你的图片
%     \caption{Schematic of the BridgeVQA construction workflow, featuring data collection, two-step problem generation, and dual-phase filtering.}
%     \label{fig:4}
% \end{figure}
\textbf{Specific Implementation of SPAI:}
% 借鉴逼近理想解排序（TOPSIS）的概念~\cite{tzeng2011multiple}，我们提出了SPAI。该方法将有效Token在全局视野下的结构差异性建模，并将其注入优势函数，以实现多样化的策略更新。具体而言，我们将轨迹视为高维空间中的向量，并计算其与“虚拟理想边界”的相对位置。
Drawing inspiration from the Technique for Order Preference by Similarity to Ideal Solution (TOPSIS)~\cite{tzeng2011multiple}, it enables the evaluation of the closeness between each trajectory and an ideal structure based on identical indicators. We propose Structural Proximity-weighted Advantage Injection, SPAI, which leverages CoT length to capture the structural differences of reward tokens from a global perspective. %SPAI injects this information into the advantage function to promote diverse policy updates. 
We provide further technical details regarding this method in Appendix \ref{sub1}.
% Specifically, the SPAI accounts for the global positional information of effective reward tokens within token-level feedback. By deriving a structural proximity score for each reward trajectory and integrating it into the vanilla advantage function, the mechanism amplifies the diversity of advantage signals.

% \begin{figure}[t]
%     \centering
%     \includegraphics[width=0.49\textwidth]{pipline.png} % 替换为你的图片
%     \caption{Schematic of the BridgeVQA construction workflow, featuring data collection, two-step problem generation, and dual-phase filtering.}
%     \label{fig:4}
% \end{figure}

In our implementation, given $R'$, we define the effective group size $G = \text{batch} \times \text{rollout}$ to ensure statistically robust trajectory modeling. 
First, to handle numerical instability in sparse reward, we apply a controlled normalization to obtain $z_{i,t}$:
%z_{i,t} = \frac{r_{i,t}}{\max\left(\sqrt{\sum_{j=1}^{G} r_{j,t}^2}, 1\right)}.
\begin{equation}
z_{i,t} = \frac{r_{i,t}}{\|\mathbf{r}_{t}\|_2 + \mathbbm{1}_{\{\|\mathbf{r}_{t}\|_2 = 0\}}},
\end{equation}
where $\|\mathbf{r}_{t}\|_2 = \sqrt{\sum_{j=1}^G r_{j,t}^2}$ and:
\begin{equation}
\mathbbm{1}_{\{\|\mathbf{r}_{t}\|_2 = 0\}} = \begin{cases} 1, & \text{if } \|\mathbf{r}_{t}\|_2 = 0 \\ 0, & \text{if } \|\mathbf{r}_{t}\|_2 > 0 \end{cases}.
\end{equation}
Next, we construct a virtual positive ideal solution ${Z}^{+}$ and negative ideal solution $Z^{-}$ by aggregating the extrema at each timestep $t$ in the trajectory sequence:
\begin{equation}
{Z}^{+} = (z_1^{+}, \ldots, z_T^{+}), \quad {Z}^{-} = (z_1^{-}, \ldots, z_T^{-}),
\end{equation}
where $z_t^{+} = \max_{i} \{z_{i,t}\}$ and $z_t^{-} = \min_{i} \{z_{i,t}\}$. Note that these ideal solutions are virtual reward boundaries, encoding the global CoT length information rather than a single  trajectory.
We then compute the euclidean distances $D_i^{+}, D_i^{-}$ between the current trajectory $i$ and these largest and smallest solutions:
\begin{equation}
D_i^{+} = \sqrt{\sum_{t=1}^{T} (z_{i,t} - z_t^{+})^2}, \quad D_i^{-} = \sqrt{\sum_{t=1}^{T} (z_{i,t} - z_t^{-})^2}.
\end{equation}
The structure relative closeness score $F_i \in (0, 1)$ is derived as:
\begin{equation}
F_i = \frac{D_i^{-}}{D_i^{+} + D_i^{-} + \epsilon},
\end{equation}
where a higher $F_i$ indicates a trajectory that is spatially closer to the optimal boundary.

% \textbf{Exploration-Enhanced Advantage Injection.}
To foster diversified exploration, we employ an asymmetric injection mechanism. Let $F_{{max}} = \max_{i} F_i$ be the maximum structural score in the current batch. We modulate the original advantage $A(\tau)$ to produce the injected advantage:
\begin{equation}
A'(\tau) = A(\tau) (1 + W_i),
\end{equation}
where the weighting factor $W_i$ is defined as:
\begin{equation}
W_i =
\begin{cases}
F_{\text{max}}, & \text{if } i \in \mathcal{I}_{\text{min}} \\
F_i, & \text{if } i \notin \mathcal{I}_{\text{min}}.
\end{cases}
\end{equation}
Here, $\mathcal{I}_{\min}$ denotes the bottom $N\%$ of samples ranked by reward $r(\tau)$. Since low-quality trajectories typically yield negative advantages ($A(\tau) < 0$), we couple their weights with the maximum spatial score $F_{\max}$ to further amplify these negative signals into a modified advantage $A'(\tau)$. This mechanism compels the policy to rapidly reallocating probability mass to accelerate exploration in higher-reward areas. To ensure stability, $N$ is chosen conservatively to target only tail-end samples that significantly deviate from expectations, thereby facilitating faster convergence while preventing oscillations in advantage estimation.

\subsection{Bidirectionally Guided Adaptive Smoothing} 
Current reward functions in agentic frameworks predominantly focus on outcome-based feedback and formatting compliance, yet they lack explicit mechanisms to incentivize or regulate deep exploration. To address this, we decompose our reward signal into three distinct components: accuracy ($R_{\text{accuracy}}$), formatting compliance ($R_{\text{format}}$), and a bidirectionally guided tool efficiency score ($R_{\text{tool\_efficiency}}$).

$R_{\text{accuracy}} \in \{0, 1\}$ denotes a binary reward indicating the correctness of the inferred answer. $R_{\text{format}}$ is divided into two parts, primarily focusing on tool formatting and summarization standards to regulate the model's behavior during inference. $R_{\text{tool\_efficiency}}$ is designed to dynamically regulate interaction depth. We propose two distinct strategic regimes to balance brevity and exhaustive search:

\textbf{Efficiency-First Regime (When $R_{\text{accuracy}}$ is Correct):} When the model's reasoning aligns with the ground truth, the reward should penalize unnecessary tool calls to prevent reasoning bloat. In this regime, we aim for a distribution with a lower mean ($\mu_1$) and broader variance ($\sigma_1$), ensuring the model maintains conciseness when it is already on the right track.

\textbf{Exploration-Driven Regime (When $R_{\text{accuracy}}$ is Incorrect):} Conversely, if the initial reasoning fails, we hypothesize that the failure stems from insufficient evidence. Here, the reward mechanism shifts to encourage deeper exploration by targeting a higher number of tool calls ($\mu_2 > \mu_1$) with a narrower focus ($\sigma_2 < \sigma_1$). This subtly guides the agent to overcome exploration myopia and retrieve diverse information in subsequent iterations. Given a scaling factor that preventing reward hacking.

To mathematically formalize this adaptive behavior and avoid optimization instability in sparse-reward settings, we leverage a gaussian kernel function as a robust structural prior for $R_{\text{tool\_efficiency}}$:\begin{equation}G(N_{\text{tool}}; \mu, \sigma) = \exp\left(-\frac{(N_{\text{tool}} - \mu)^2}{2 \sigma^2}\right).\end{equation}The parameters are dynamically switched between $\Theta_1 = (\mu_1, \sigma_1)$ and $\Theta_2 = (\mu_2, \sigma_2)$ based on the correctness of $R_{\text{accuracy}}$. These distributions serve not as hyperparameters for fine-tuning, but as fixed inductive biases (\cref{tab:4}) that enable the agent to calibrate interaction depth via feedback.

Finally, these components are aggregated into the global reward function to ensure that accuracy remains the primary optimization objective:
\begin{equation}R_{\text{final}} = 0.7 \cdot R_{\text{accuracy}} + 0.2 \cdot R_{\text{format}} + 0.1 \cdot R_{\text{tool\_efficiency}}.
\label{eq:reward_final}
\end{equation}

% \begin{figure}[ht]
%     \centering
%     \includegraphics[width=0.49\textwidth]{pipline.png} % 替换为你的图片
%     \caption{Schematic of the BridgeVQA construction workflow, featuring data collection, two-step problem generation, and dual-phase filtering.}
%     \label{fig:4}
% \end{figure}
% \subsection{Construction Dataset and Updatable Sandbox}
\subsection{The Framework of Deepening Reasoning MMSearch }
\begin{figure*}[t]
    \centering
    \includegraphics[width=1.0\textwidth]{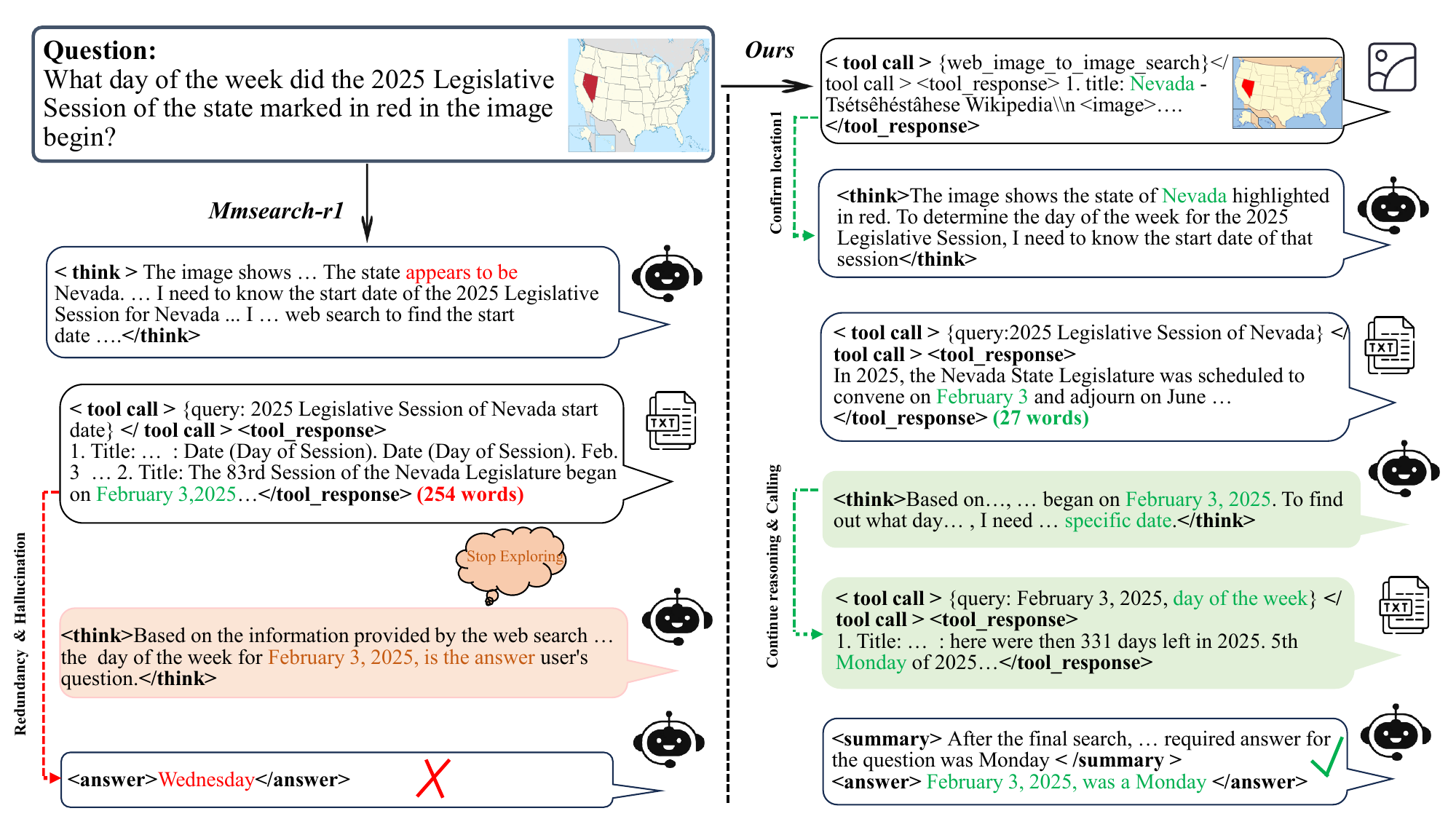} % 替换为你的图片
    \caption{The comparison of multi-step interaction. Unlike traditional baseline models that exhibit short-sighted exploration and redundant trajectories, our framework enables stable and reliable interactions. It achieves further performance improvements while increasing the number of model interactions and reducing their corresponding lengths.}
    \label{fig:4}
\end{figure*}
% \begin{figure}[tbp]
%     \centering
%     \includegraphics[width=0.49\textwidth]{pipline.png} % 替换为你的图片
%     \caption{Schematic of the BridgeVQA construction workflow, featuring data collection, two-step problem generation, and dual-phase filtering.}
%     \label{fig:4}
% \end{figure}
Current knowledge-intensive training sets often suffer from limited search depth; for instance, the FVQA dataset in MMSearch-R1 was primarily designed for a maximum of two-turn tool interactions. Furthermore, these datasets frequently lack temporal relevance. To bridge this gap, we introduce BridgeVQA, a deep multi-hop VQA dataset sourced from the latest web information. Comprising 3,602 questions (detailed in Appendix \ref{sub3} and \ref{sub4}), BridgeVQA presents a significantly higher difficulty level than FVQA. Developed through a visual anchor pipeline, it integrates cross-modal dependencies and CoT mechanisms to address high-order reasoning challenges. 
% Additionally, we establish a scalable and cost-effective tool environment utilizing offline image caching and text chunking, as described in Appendix \ref{sub5}. Finally, a compression-refinement agent is developed to mitigate informational redundancy during tool retrieval.

% First, we propose BridgeVQA, a multi-hop VQA dataset comprising 3,602 questions detailed in Appendix \ref{sub3} and \ref{sub4}. Built via a vision-anchor pipeline, it incorporates cross-modal dependencies and CoT mechanisms \cite{wei2022chain} to tackle high-order reasoning. Second, we established a scalable and cost-effective tool environment using offline image caching and text chunking as illustrated in Appendix \ref{sub5}. Finally, we developed a compression-refinement agent to mitigate redundancy in tool retrieval.

\textbf{The Pipline of BridgeVQA:} As shown in \cref{fig:3}, we first systematically leverage the latest news website resources~\cite{fu2025livevqa} to obtain image-text pairs. Then we perform visual entity anchoring. We locally deploy Qwen2.5-VL-72B-Instruct~\cite{bai2025qwen2} to analyze acquired images and generate visual entity anchors. These anchors primarily include key events, persons, or specific objects within the images, serving as visual cues for subsequent question construction. %This forms the basis for complex dataset generation.

To guarantee an implicit reasoning path of at least three hops, we employ a structured question generation formula guided by strict logical constraints. First, the core query structure follows the triple dependency chain: \text{Query (Anchor A)} \(\xrightarrow{\text{Relation } \mathcal{R}}\) \text{Bridge B} \(\xrightarrow{\text{Dependency}}\) \text{Target Fact C}. This translates into the template: What is the [Target Fact C] of the [Bridge Entity B] that has a [Relation $\mathcal{R}$] with the [Visual Anchor A] ? Then for complex multi-hop VQA instance generation. LLMs are utilized to generate preliminary question-answer pairs based on the formula. Crucially, each instance is paired with a reliable, interactive CoT trace, explicitly documenting the $\text{A} \rightarrow \text{B} \rightarrow \text{C}$ reasoning path.

We established a rigorous data curation pipeline to extract high-quality instances. Initially, a preliminary screening was conducted using the model, where generated CoT sequences served as a basis for logical validation. Subsequently, for ambiguous cases, manual filtering and refinement were performed to resolve uncertainties. This pipeline ultimately yielded approximately 3,602 complex multi-hop instances, ensuring high logical consistency and verifiability.

%Additionally, we establish a scalable and cost-effective tool environment utilizing offline image caching and text chunking, as described in Appendix \ref{sub5}. 

% \textbf{Reliable Local Tool Service:} As in \cref{fig:2}, to ensure the dataset remains tool-retrievable, we chunked the collected textual knowledge and stored it in a knowledge base as a search source, integrating new information with offline Wikipedia data. Simultaneously, to reduce costs and enhance training efficiency, we implemented pre-fetching for online image searches to serve as a cache. This viable workflow guarantees the sustainable updateability of the training data.
\definecolor{tablegray}{gray}{0.85}
\begin{table*}[ht]
\centering
\caption{Main results across five benchmarks. 
\textsuperscript{\S} \;Since this study utilizes the latest official version of LiveVQA, results from earlier versions used by WebWatcher are not strictly comparable.}
\label{tab:1}
\newcolumntype{Y}{>{\centering\arraybackslash}X} % 定义一个居中的可伸展列
\begin{tabularx}{\textwidth}{l | Y | YYYYY}

\toprule
\textbf{Model}& \textbf{Average} & \textbf{FVQA-test} & \textbf{InfoSeek} & \textbf{SimpleVQA} & \textbf{LiveVQA} & \textbf{MMSearch} \\
\midrule

\rowcolor{tablegray}
\multicolumn{7}{c}{\textit{Direct}} \\
GPT-4o        & 36.0 & 41.7  & 42.7  & 46.6  & 26.9  & 22.2\\
Gemini2.5 Pro  & 36.4& 37.2  & 37.0  & 53.4  & 27.7  & 26.9  \\
Qwen2.5-VL-7B   & 18.1 & 20.9 & 23.9 & 30.4 & 8.3 & 7.2  \\
Qwen2.5-VL-32B  & 25.0  & 24.7 & 25.8 & 40.1 & 18.7 & 15.7   \\
\midrule

\rowcolor{tablegray}
\multicolumn{7}{c}{\textit{Search Agent}} \\
Qwen2.5-VL-7B   & 26.0 & 34.2  & 28.3  & 35.8 & 10.7 & 21.1   \\
Qwen2.5-VL-32B  & 38.0  & 51.3  & 38.0  & 48.5 & 24.8 & 27.3  \\
MMSearch-R1   & 50.6 & 58.4    & 55.1    & 57.4 & 28.3 & 53.8 \\
DeepMMSearch-R1 & ---& ---    & 47.5    & 55.9 & --- & --- \\
WebWatcher\textsuperscript{\S} & ---& ---    & ---    & 54.3 & ---& 55.3 \\
Deepeyes2    & ---& 60.6  & 51.1  & 59.4 & --- & \textbf{63.7}   \\
Ours  & \textbf{57.6} & \textbf{66.8}    & \textbf{60.0}   & \textbf{63.2} & \textbf{37.0} & 61.0\\
\bottomrule
\end{tabularx} % 结束标签需对应为 tabular*
\end{table*}

% \textbf{Compressed Refining Agent:} 
% 如图4所示，该模型在与工具的展开交互过程中生成了许多轨迹，特别是在文本搜索产生冗长响应时。我们通过设置了top3时，使用该方法能够压缩80%的冗余信息的同时，还正确返回了结果。同时，在进行推理时，使用refining agent之后，同等条件下模型的推理轮数和性能也获得了提升。
% In \cref{fig:4}, the model generates extensive trajectories during tool interaction, particularly when text searches yield prolix responses. As illustrated in \cref{fig:2} and Appendix \ref{sub2}, by setting the retrieval to the top-3 results, our method successfully eliminates $80\%$ of redundant information while maintaining correct outcomes. Furthermore, the integration of the refining agent during reasoning significantly enhances both the interaction depth and overall performance under equivalent conditions. This approach significantly shortening the model's window context length:
% \begin{equation}
% o_{refine} = F(p, q, r),
% \end{equation}
% where $p$ represents the prompt template extraction $q$ denotes the sub-question retrieved during model interaction, while $r$ signifies redundant results returned by the tool service system. 

\textbf{Compressed Refining Agent:} In \cref{fig:4}, we observe that text searches often yield prolix responses, leading to extensive interaction trajectories. To address this, we restrict retrieval to the top-3 results (see \cref{fig:2}, Appendix \ref{sub2}), successfully eliminating $80\%$ of redundant information without compromising accuracy. Additionally, integrating a refining agent enhances reasoning depth and performance while substantially shortening the model's context window. This process is formulated as:
\begin{equation}
o_{refine} = F(p, q, r),
\end{equation}
where $p$, $q$, and $r$ denote the prompt template, the retrieved sub-question, and the tool-system's redundant output, respectively.
\section{Experiments}
\label{sec:exp}

\subsection{Experimental Setup}
\textbf{Implementation Details:} 
We employed the veRL~\cite{sheng2025hybridflow} framework for reinforcement learning, utilizing Qwen2.5-VL-7B as the backbone model and GRPO for optimization. Specific configuration details include a batch size of 128, with 8 rollouts per problem and temperature is set to 0. The maximum interaction rounds were set to 14, and the model was trained for 2 epochs. For evaluating rewards based on training outcomes, we deployed Qwen3-14B~\cite{yang2025qwen3} locally. Our method employs low-cost local tool services during the training phase, while utilizing online search services for testing. The specific system prompts, judge prompts, and summarization prompts are presented in Appendix \ref{sub7}. To facilitate reproducibility, we have provided the complete source code and datasets in the supplementary materials.
%the maximum response length to 26384, 

\begin{figure}[t]
    \centering
    \begin{subfigure}[b]{0.23\textwidth}  % 稍微调大一点
        \centering
        \includegraphics[width=\linewidth]{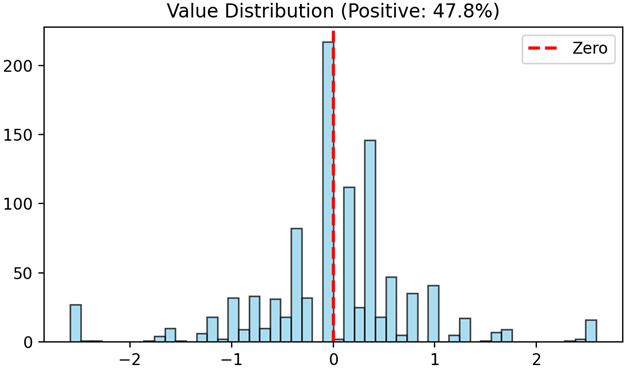}
        \label{fig:a}
        \vspace{-10pt} 
        \caption{}
    \end{subfigure}
    \hfill
    \begin{subfigure}[b]{0.23\textwidth}  % 稍微调大一点
        \centering
        \includegraphics[width=\linewidth]{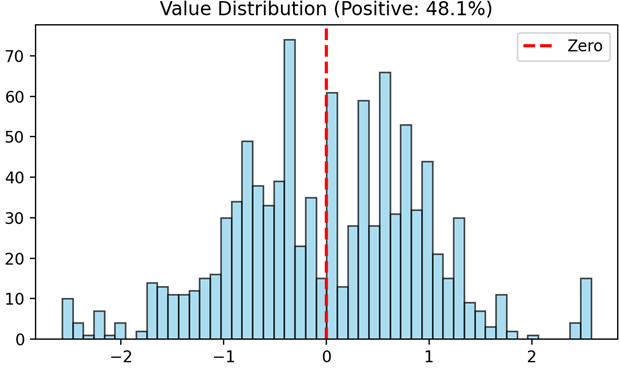}
        \label{fig:b}
        \vspace{-10pt} 
        \caption{}
    \end{subfigure}
    \caption{Advantage distribution in a training batch of baseline and our framework.}
    \label{fig:5}
\end{figure}
% \textbf{Tool Details.} 
% During both training and evaluation, we equip the model with two distinct tools from our tool sandbox: a text search tool and an image search tool. All  tool interactions strictly adhere to the JSON schema specified by the OpenAI function-calling standard.

% (1) Text search tool details. To ensure stable, reproducible training, prevent training experiment failures due to external tool unreliability, and mitigate the high costs, we employ the local search tool within our tool sandbox during the RL training phase. For the evaluation phase, we transition to the web search service within our tool sandbox. 
% % This tool returns the top-5 most relevant results for the given query. We emphasize that no web scraper or summarization model is employed. The model is directly provided with the raw Serper API Search output, which consists of the 'title', 'link', and 'snippet' for each result.

% (2) Image search tool details.
% For image search, we utilize the pre-computed image cache within the tool sandbox. The tool ultimately returns the top-k images most similar to the image of the VQA sample, along with their web titles. Further details regarding the search tool configurations for training and testing are provided in the Supplementary Materials.
% % Operationally, the tool retrieves the corresponding results by querying the cache with the current sample's unique data ID. 
% % At runtime, when the model calls the image search tool, the top-3 cached results will be returned to the model.

\textbf{Datasets and Benchmarks:}
For training data, we utilized the open-source FVQA-train dataset followed by our proposed deep exploration dataset BridgeVQA.

To ensure a fair and consistent comparison, our model and all baseline methods are evaluated on the same benchmarks, which collectively span knowledge-intensive Visual Question Answering (VQA) and multimodal retrieval-based reasoning challenges. These benchmarks include: FVQA-test~\cite{wu2025mmsearch}, Infoseek~\cite{chen2023can}, SimpleVQA~\cite{cheng2025simplevqa}, MMsearch ~\cite{jiang2024mmsearch}, and LiveVQA~\cite{fu2025livevqa}.
In the LiveVQA experiments, the latest official versions were used throughout to ensure a fair comparison. For specific datasets, can view it in the Appendix \ref{sub6}.

To validate the generalization of our proposed SPAI method to other non-search-based visual tasks, we referred to the evaluation system in~\cite{zhu2025shuffle}. To ensure comparability and fairness, we evaluate all models under a consistent set of evaluation protocols. we applied it to MathVerse~\cite{zhang2024mathverse}, MathVision~\cite{wang2024measuring}, MathVista~\cite{lu2023mathvista}, and HallusionBench~\cite{guan2024hallusionbench} and ChartQA~\cite{masry2022chartqa}.

% We evaluate our model's performance on five benchmarks spanning knowledge-intensive visual question answering (VQA) and multimodal retrieval-based reasoning: FVQA-test\cite{wu2025mmsearch}, Infoseek\cite{chen2023can}, SimpleVQA\cite{cheng2025simplevqa}, MMsearch\cite{jiang2024mmsearch}, LiveVQA\cite{fu2025livevqa}. Following MMSearch-R1, we utilize the full FVQA-test set. Specifically, for Infoseek, we selected a subset of 2,000 samples. For SimpleVQA, we selected 1,013 QA pairs containing only English. For MMSearch, we selected 171 samples that include images. For the LiveVQA benchmark, we note a significant discrepancy in the data's version history. While the MMSearch-R1 paper reports results on a 3,602-sample version, this version appears to be deprecated or no longer publicly accessible. We use the currently available 2,384-sample version.
\definecolor{tablegray}{gray}{0.85}

\begin{table*}[htbp]
\centering
\caption{Comparison of our method with other visual tasks.}
\label{tab:5}
% 使用 tabular* 并设置宽度为 \textwidth
% @{\extracolsep{\fill}} 会自动填满列间距
%\begin{tabular*}{\textwidth}{@{\extracolsep{\fill}}lccccc}
\newcolumntype{Y}{>{\centering\arraybackslash}X} % 定义一个居中的可伸展列
\begin{tabularx}{\textwidth}{l | Y | YYYYY}
\toprule
\textbf{Model} & \textbf{Average} & \textbf{Mathvista} & \textbf{Mathverse} & \textbf{Mathvision} & \textbf{ChartQA} & \textbf{HallBench} \\
\midrule

\rowcolor{tablegray}
\multicolumn{7}{c}{\textit{Close-source Models}} \\
GPT-4o     & ---    & 63.8  & 50.8  & 30.4  & ---  & 55.0  \\
Claude-3.7-Sonnet & ---  & 66.8  & 52.0  & 41.3  & ---  & 55.4 \\
\midrule

\rowcolor{tablegray}
\multicolumn{7}{c}{\textit{Open-source Models}} \\
Qwen2.5-VL-7B & 55.9   & 68.1 & 41.1  & 25.4 & 79.8 & 65.2  \\
OpenVLThinker-7B & 55.7 & 69.7  & 46.4  & 24.8 & 78.4 & 59.1   \\
WeThink-7B & --- & 70.9    & 44.7    & 27.2 &  ---  & 55.1 \\
ThinkLite-VL-7B & 59.6 & 72.4    & 45.2    & 28.0 & 82.0 & 70.2\\
VL-Rethinker-7B & --- & 73.7   & ---   & 28.4 & 79.0& 69.9 \\
Ours (w/ SPAI) & \textbf{61.6}  & \textbf{74.2}    & \textbf{50.6}   & \textbf{29.2} & \textbf{82.1} & \textbf{72.0} \\
\bottomrule
\end{tabularx} % 注意这里也要改成 tabular*

%\vspace{0.3em}

\end{table*}

\textbf{Baselines:}
(1) The initial evaluation paradigm employed is direct inference, where the model operates exclusively without access to external callable tools. In this configuration, the model is constrained to rely entirely on its internal parametric knowledge for question answering. We utilized this specific setting to establish a performance benchmark for several state-of-the-art models, including GPT-4o~\cite{openai2024gpt4ocard}, Gemini-2.5 Pro~\cite{team2023gemini}, and the open-source Qwen2.5-VL series. (2)
 The second approach involves training intelligent agents using methods such as SFT and RL. Both methods utilize image search and text search tools. Note that our training employs low-cost local tool services. It is worth noting that the agent WebWatcher~\cite{geng2025webwatcher} uses more paid tools (such as visit, code)

\textbf{Metrics:} For fair comparison, we employ the MMSearch-R1 methodology, utilizing the LLM-as-Judge framework as an evaluation metric to measure the accuracy of model responses. Specifically, we deploy Qwen2.5-72B-Instruct locally as the judging model, which determines correctness by receiving the original image, question, ground truth answer, and model response.

\begin{table}[h!]
\centering
\caption{Ablation experiments in different components.}
\label{tab:2}
% 1. 设置列间距。由于是5列，建议先尝试 6.5mm，若未铺满则调大，若溢出则调小。3.7
\setlength{\tabcolsep}{2.4mm} 
{
    % 2. 保持原有居中排版 cccc|c
    \begin{tabular}{cccc|c} 
    \toprule
    \textbf{BridgeVQA} & \textbf{SPAI} & \textbf{BGAS} & \textbf{Refining} & \textbf{Acc} \\
    \midrule
     - & - & - & - & 57.7\% \\
    \checkmark & - & - & - & 59.7\% \\
    \checkmark & - & \checkmark & - & 60.8\% \\
    \checkmark & \checkmark & - & - & 63.4\% \\
    \checkmark & \checkmark & \checkmark & - & 64.9\% \\
    \rowcolor{gray!10} % 如果需要保持与前表一致的样式，可以给最后一行加底色
    \checkmark & \checkmark & \checkmark & \checkmark & \textbf{66.8\%} \\
    \bottomrule
    \end{tabular}
}
\end{table}
% \begin{table}[h!]
% \centering
% \caption{Ablation experiments in different components.}
% \label{tab:2}
% % 1. 去掉了 \resizebox{...} 的包裹
% % 2. 添加了 \small 命令将字体变小 (也可以用 \footnotesize)
% \small 
% % 3. 稍微增加一点行高，让表格看起来不那么拥挤(可选)
% \renewcommand{\arraystretch}{1.2} 
% \begin{tabular}{cccc|c} 
% \toprule
% \textbf{Data} & \textbf{SPAI} & \textbf{AGR} & \textbf{Extract} & \textbf{Acc} \\
% \midrule
%  - & - & - & - & 57.7\% \\
% \checkmark & - & - & - & 59.7\% \\
% \checkmark & - & \checkmark & - & 60.8\% \\
% \checkmark & \checkmark & - & - & 63.4\% \\
% \checkmark & \checkmark & \checkmark & - & 64.9\% \\
% \checkmark & \checkmark & \checkmark & \checkmark & \textbf{66.8\%} \\
% \bottomrule
% \end{tabular}
% \end{table}

% 请确保导言区有：\usepackage{tabularx}
% 请确保导言区有：\newcolumntype{Y}{>{\centering\arraybackslash}X}

\begin{table}[h!]
\centering
\caption{Comparison of our method and others in different visual datasets.}
\label{tab:3}
% 调整列间距命令，7mm 可以根据实际视觉效果增减
\setlength{\tabcolsep}{4.2mm} 
{
    \begin{tabular}{l|c|c|c}
    \toprule
    \textbf{Setting} & \textbf{Fvqa-test} & \textbf{HallBench} & \textbf{Step} \\
    \midrule
    \rowcolor{gray!10}
    \texttt{Ours} & \textbf{66.8\%} & \textbf{72.0\%} & 160 \\
    \texttt{GRPO} & 57.7\% & 65.2\% & 200 \\
    \texttt{DAPO} & 64.0\% & 68.7\% & 225 \\
    \bottomrule
    \end{tabular}
}
\end{table}
% \begin{table}[h!]
% \centering
% \caption{Comparison of our method and others in different visual datasets.}
% \begin{tabular}{l|c|c|c}
% \toprule
% \textbf{Setting} & \textbf{Fvqa-test} & \textbf{HallBench} & \textbf{Step} \\
% \midrule
% \rowcolor{gray!10}
% \texttt{Ours} & \textbf{66.8\%} & \textbf{72.0\%} & 160 \\
% \texttt{GRPO} & 57.7\% & 65.2\% & 200 \\
% \texttt{DAPO} & 64.0\% & 68.7\% & 225 \\
% \bottomrule
% \end{tabular}
% \label{tab:3}
% \end{table}

\begin{table}[h!]
\centering
\caption{Ablation experiments with different specific parameter configurations. SPAI Inject represents negative sample replacement ratio of Structural Proximity-weighted Advantage Injection mechanism, BGAS represents Bidirectionally Guided Adaptive Smoothing.}
\label{tab:4}
\begin{tabular}{l|c|c|c}
\toprule
\textbf{SPAI Inject} & \textbf{BGAS $(\mu, \sigma)$} & \textbf{Acc} & \textbf{mean turns} \\
\midrule
\rowcolor{gray!10}
\texttt{5\%}  & (4, 1.2) & \textbf{66.8\%} & 9   \\
\texttt{5\%}  & (3, 1.2) & 65.2\%          & 7 \\
\texttt{5\%}  & (0, 0)    & 63.4\%          & 10  \\
\texttt{0\%}  & (3, 1.2) & 63.2\%          & 7   \\
\texttt{10\%} & (4, 1.2) & 65.3\%          & 9   \\
\bottomrule
\end{tabular}
\end{table}

\begin{figure}[h!]
    \centering
    
    \begin{subfigure}[b]{0.23\textwidth}
        \centering
        \includegraphics[width=\linewidth]{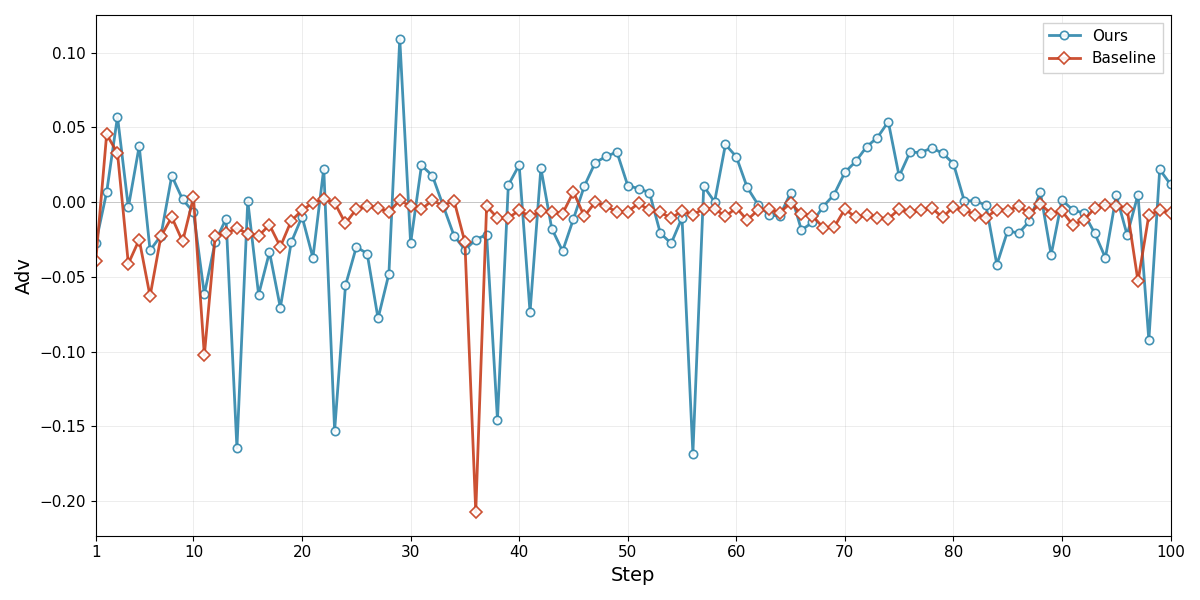}
        \caption{}
        \label{fig:comparison}
    \end{subfigure}
    \hfill
    \begin{subfigure}[b]{0.23\textwidth}
        \centering
        \includegraphics[width=\linewidth]{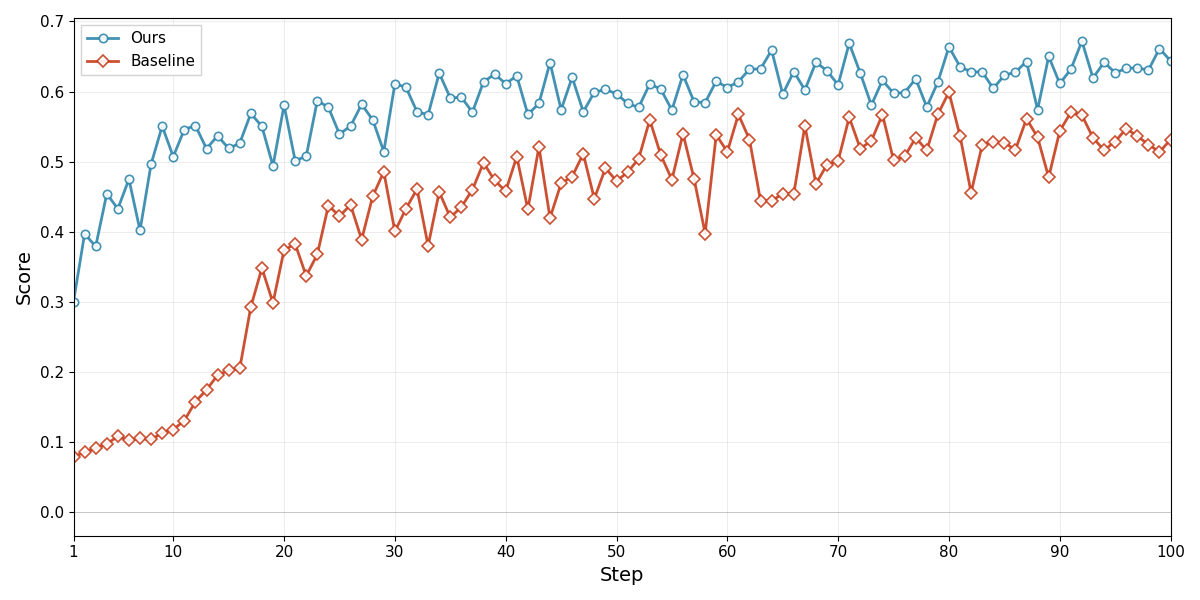}
        \caption{}
        \label{fig:difference}
    \end{subfigure}

    \begin{subfigure}[b]{0.23\textwidth}
        \centering
        \includegraphics[width=\linewidth]{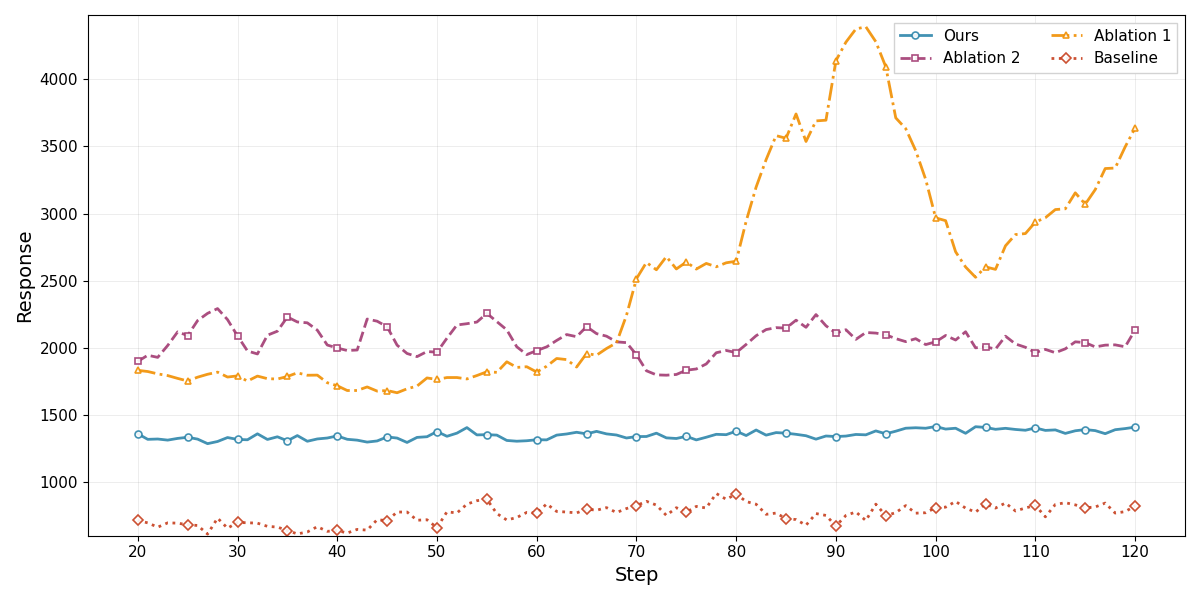}
        \caption{}
        \label{fig:distribution}
    \end{subfigure}
    \hfill
    \begin{subfigure}[b]{0.23\textwidth}
        \centering
        \includegraphics[width=\linewidth]{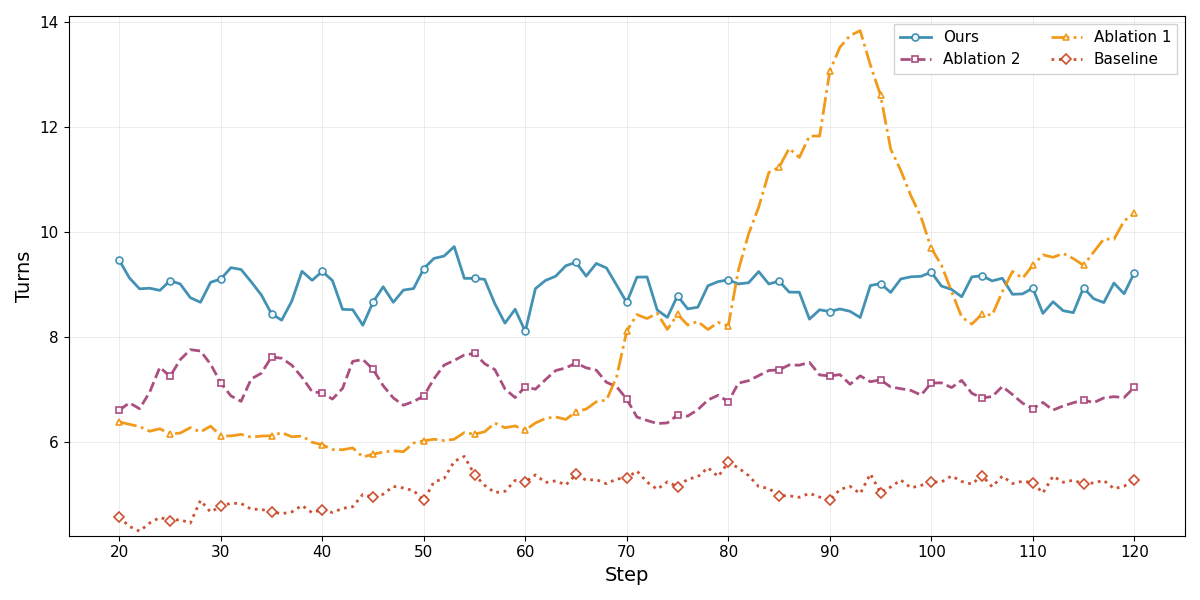}
        \caption{}
        \label{fig:cumulative}
    \end{subfigure}
    
    \caption{(a) Advantage distribution of GRPO and ours in FVQA. (b) Training score of GRPO and ours in FVQA. (c) Response length of GRPO and ours in FVQA. (d) Turns of GRPO and ours in Geo3k.}
    \label{fig:6}
\end{figure}

\subsection{Main Results} 
\textbf{Evaluation on Search-oriented Bench:} 
To ensure robustness, \cref{tab:1} reports the average performance across three independent inference trials. Our method significantly outperforms both open-source and closed-source models relying on direct reasoning. These results indicate that direct reasoning is insufficient for this task, as its knowledge-intensive nature requires the integration of external knowledge and multi-turn interactive reasoning to generate accurate solutions.

Our model achieves SOTA performance compared to other trained agents. For instance, it outperforms MMsearch-R1 by $8.4\%$ on FVQA-test and $4.9\%$ on InfoSeek. These gains are attributed to our mechanism's ability to enhance model exploration, enabling it to discover more complex samples to resolve challenging queries. Notably, our RL-only approach achieves a substantial improvement even over DeepEyes2, which utilizes a combined SFT and RL strategy; this underscores our method's capacity to reduce reliance on costly supervised data. Furthermore, in comparison to the base model Qwen2.5-VL-7B, our framework yields a remarkable improvement of $32.6\%$ on the FVQA-test dataset, further validating the high efficiency of our training framework. Regarding the training environment, while contemporary approaches rely on knowledge-intensive online Google Search tools, our method utilizes cost-effective local tool services and still achieves state-of-the-art (SOTA) performance.

\textbf{Evaluation Other Visual Bench:} 
To demonstrate the applicability of SPAI to other non-search visual tasks, we compared it against representative benchmarks in mathematics, charts, and perception. As shown in \cref{tab:5}, our method achieves significant performance improvements over baseline models. For example, achieving a $9.5\%$ improvement on the Mathverse dataset, while also demonstrating advantages over most existing methods~\cite{yang2025wethink,wang2025sota,wang2025vl,deng2025openvlthinker,hong2025deepeyesv2} in current comparisons, outperforming the majority of trained approaches.

\subsection{Ablation Study} 
\textbf{Component Ablation:}
As shown in \cref{tab:2} and  \cref{fig:6} (b), (c), and (d), we conducted detailed ablation studies on the proposed framework. Overall, the performance of each module demonstrates steady improvement, accompanied by an effective increase in interaction turns and a significant optimization of response quality. These results successfully address the two major limitations previously identified: the framework not only mitigates exploration myopia and trajectory redundancy but also significantly enhances model performance, demonstrating the overall effectiveness of the proposed system.

Quantitatively, the multi-hop dataset improved baseline performance by $2\%$, underscoring its contribution to reasoning depth. Subsequent ablation of SPAI and BGAS mechanisms showed incremental gains of $3.7\%$ and $1.5\%$, respectively, validating the efficiency of the integrated framework. Furthermore, the combination of our dataset and mechanisms encouraged more profound exploration. Although the reward mechanism slightly distilled the response length and turn count, the results surpassed the baseline across all metrics. This step-by-step evaluation demonstrates our method's ability to balance robust exploration with redundancy control.
% Regarding performance, compared to the baseline, the utilization of the new multi-hop dataset achieved a $2\%$ improvement, indicating that the dataset strengthens the model's reasoning capabilities. By incorporating the SPAI and BGAS mechanisms, the model’s performance further increased by $3.7\%$ and $1.5\%$, respectively, proving the high efficiency of the integrated mechanism. Concerning response length and interaction turns, the combination of the new dataset and the advantage injection mechanism facilitated deeper exploration. Although the reward mechanism slightly reduced the response length and the number of turns, both metrics remain superior to the baseline. This step-by-step validation confirms that our method achieves enhanced exploration while effectively avoiding excessive redundancy.

\textbf{Compared with Other RL algorithms:}
As shown in \cref{tab:3}, under the same configuration, our method achieves higher performance with fewer steps compared to other classical reinforcement learning methods. This demonstrates that the framework not only enhances performance but also accelerates exploration, effectively resolving the exploration degeneration dilemma in the current task.

% \textbf{Algorithm Effectiveness.} 
\textbf{Experiment on Advantageous Signals:}
As shown in  \cref{fig:5}, the SPAI mechanism amplifies both positive and negative advantage values by incorporating spatial trajectory features, thereby enhancing the diversity of advantage signals. This leads to a more dispersed distribution of advantages, enabling the model to explore more effectively and avoid exploration collapse. These results demonstrate that our modeling of spatial reward trajectories successfully achieves diversified exploration. Furthermore, as illustrated in \cref{fig:6} (a), our method ensures a sufficient supply of advantage signals throughout the training process, effectively preventing signal stagnation.

\textbf{Sensitivity Analysis and Hyperparameter Robustness:} As presented in \cref{tab:4}, we conducted an ablation study to investigate the sensitivity of the core hyperparameters in SPAI and BGAS. First, impact of SPAI Injection Ratio $N\%$: we observe an inverted U-shape performance trend regarding the negative replacement ratio $N\%$. At $N=0\%$, the model lacks sufficient penalty for suboptimal paths. The optimal performance is achieved at $N=5\%$, which effectively suppresses the tail of poor trajectories. Increasing $N$ to $10\%$ leads to a performance drop. This confirms our hypothesis that treating too many samples as hard negatives, disrupting the learning of potentially valid reasoning paths. This suggests that $N$ should be set to target statistical outliers rather than a substantial portion of the batch.

Second, rationale for BGAS Parameters $\mu, \sigma$: our comparison of $\mu, \sigma$ configurations highlights the trade-off between exploration and efficiency. Regarding the comparison between static and dynamic settings, the dynamic adjustment strategy shown in Row 1 significantly outperforms the static configurations listed in Rows 2 and 4. In terms of parameter sensitivity, setting $\mu$ too low for incorrect samples fails to encourage sufficient information retrieval. As evidenced by Row 4, this limitation leads to under-search failures. Conversely, applying overly loose constraints with a high $\sigma$ in failure modes results in redundant tool calls, causing over-search issues that yield no gains in accuracy.

\textbf{Comparison with Other Visualization Methods:}
As revealed by the visualization in \cref{fig:1}, our method not only effectively alleviates the repetition rate but also mitigates the rigidity of identical-reward trajectories. Furthermore, as demonstrated in  \cref{fig:4}, the proposed framework successfully addresses the current issue of exploration degeneration.
\section{Conclusion}
\label{sec:conclusion}

% Our proposed DR-MMSearchAgent framework effectively addresses the exploration degeneration dilemma. Trained on datasets constructed from internet-based visual information, the approach achieves state-of-the-art performance on knowledge-intensive VQA benchmarks.

Our proposed DR-MMSearchAgent framework effectively addresses the exploration degeneration dilemma. The framework leverages the SPAI mechanism to achieve global modeling of verification reward, thereby facilitating diversified exploration. Simultaneously, it employs the BGAS mechanism to dynamically adjust interaction tolerance, ensuring information reliability. Trained on datasets constructed from internet-based visual information, the approach achieves SOTA performance on knowledge-intensive VQA tasks.
% while demonstrating strong generalization capabilities across fundamental visual tasks.
% outperforming MMsearch-R1 by 8.4$\%$ on the FVQA-test dataset

\newpage
\section*{Impact Statement}
1. This paper presents DR-MMAgent, a framework designed to enhance the depth and reliability of multimodal search agents. 2. The datasets used in this paper, such as BridgeVQA, are sourced from publicly accessible resources, including up-to-date news outlets. Strict data processing procedures have been implemented to ensure logical consistency and full compliance with relevant privacy protection and ethical standards, with any personally sensitive information excluded.
\bibliography{icml2026}
\bibliographystyle{icml2026}

%%%%%%%%%%%%%%%%%%%%%%%%%%%%%%%%%%%%%%%%%%%%%%%%%%%%%%%%%%%%%%%%%%%%%%%%%%%%%%%
%%%%%%%%%%%%%%%%%%%%%%%%%%%%%%%%%%%%%%%%%%%%%%%%%%%%%%%%%%%%%%%%%%%%%%%%%%%%%%%
% APPENDIX
%%%%%%%%%%%%%%%%%%%%%%%%%%%%%%%%%%%%%%%%%%%%%%%%%%%%%%%%%%%%%%%%%%%%%%%%%%%%%%%
%%%%%%%%%%%%%%%%%%%%%%%%%%%%%%%%%%%%%%%%%%%%%%%%%%%%%%%%%%%%%%%%%%%%%%%%%%%%%%%
\newpage
\appendix
\twocolumn
% \clearpage
% \setcounter{page}{1}

% \section{THEORETICAL ANALYSIS}
% \label{sec:rationale}
% 
% Having the supplementary compiled together with the main paper means that:
% % 
% \begin{itemize}

% \item The main paper can forward reference sub-sections within the supplementary explicitly (e.g. referring to a particular experiment); 
% \item When submitted to arXiv, the supplementary will already included at the end of the paper.
% \end{itemize}
% % 
% To split the supplementary pages from the main paper, you can use \href{https://support.apple.com/en-ca/guide/preview/prvw11793/mac#:~:text=Delete%20a%20page%20from%20a,or%20choose%20Edit%20%3E%20Delete).}{Preview (on macOS)}, \href{https://www.adobe.com/acrobat/how-to/delete-pages-from-pdf.html#:~:text=Choose%20%E2%80%9CTools%E2%80%9D%20%3E%20%E2%80%9COrganize,or%20pages%20from%20the%20file.}{Adobe Acrobat} (on all OSs), as well as \href{https://superuser.com/questions/517986/is-it-possible-to-delete-some-pages-of-a-pdf-document}{command line tools}.

% This theoretical analysis provides intuitive support for the effectiveness of our proposed method by addressing the issue of short-sightedness in model exploration under tool-invocation tasks. Specifically, by matching reward scale invariance with the original advantage computation, the theory not only preserves the original optimal policy but also achieves greater discrimination between trajectories with identical rewards, thereby accelerating model exploration.

\section{Theoretical Foundations}
\label{sub1}

% 在本节中，我们将展示我们的SPAI机制和GRPO具有奖励尺度不变性，能够在token级别奖励和整体奖励实现异构奖励有效融合。同时对该机制进行深入分析，介绍我们的方法如何在空间方面实现优势信号的多样性探索。
% In this section, we demonstrate that our SPAI mechanism and GRPO exhibit reward-scale invariance, enabling the effective fusion of heterogeneous token-level and holistic rewards. Furthermore, we provide an in-depth analysis of this mechanism, illustrating how our approach facilitates diverse exploration of advantage signals across the spatial dimension.

\subsection{Preliminaries}
% --- LaTeX Copy-Paste Version Start ---
% The token-level reward $r_{i,t}$ is defined based on a reward matrix $\mathbf{R} \in \mathbb{R}^{G \times T}$ initialized with all zeros, where $G$ denotes the group size and $T$ represents the maximum sequence length. For each trajectory $\tau$ with a valid output length $L_i$, the sequence-level reward ${r}(\tau)$ is assigned exclusively to the final valid token:
% \begin{equation}
% r_{i,t} = 
% \begin{cases} 
% {r}(\tau), & \text{if } t = L_i \\
% 0, & \text{if } 1 \leq t < L_i.
% \end{cases}
% \end{equation}
% Notably, our experimental observations indicate a high prevalence of the identical reward, diverse trajectories phenomenon, where multiple distinct reasoning paths yield the same reward score. Such redundancy in advantage signals, if left unmanaged, can lead to a decline in exploration diversity and a subsequent increase in signal homogenization.

% --- LaTeX Copy-Paste Version End ---

Let a trajectory be denoted by $\tau$, and let its (trajectory-level) return be:
\begin{equation}
R(\tau)=\sum_{t=1}^T r_t(\tau),
\end{equation}
where $r_t(\tau)$ denotes the reward value of the last response token in the current trajectory.
We denote the empirical mean and standard deviation of returns within a group by
\(\mu=mean(R)\) and \(\sigma=std(R)\).
The standardized GRPO advantage is defined as
\begin{equation}
\label{subgs1}
A(\tau) \;=\; \frac{R(\tau)-\mu}{\sigma}.
\end{equation}

And computes a SPAI score $W_i\in(0,1)$.
The injected advantage used in updates is
\begin{equation}
% A'(\tau) =A (\tau)+ A (\tau)\cdot F(\tau).
A'(\tau) =A (\tau)+ A (\tau)\cdot W_i.
\end{equation}

\subsection{The Proof of Reward-scale Invariance}
% 总的来说，构建了一个量纲解耦的决策流形，分别负责在高维空间中进行几何定位，在分布层面进行整体的优势计算。这种双重尺度不变性确保了异构奖励在融合过程中，梯度的物理含义不会被数值的绝对尺度所污染。
In summary, this approach constructs a dimensionally decoupled decision manifold, which facilitates geometric localization within high-dimensional spaces and enables advantage estimation at the distributional level. This dual scale-invariance ensures that during the fusion of heterogeneous rewards, the physical significance of gradients remains uncontaminated by the absolute scales of numerical values.

First, show that the standardized advantage $A(\tau)$ is \emph{exactly invariant}
under affine transformations of the returns of the form $R'(\tau) = \alpha R(\tau)$
with $\alpha>0$.

\begin{lemma}Invariance of standardized advantage under affine transforms.
Let $R'(\tau)=\alpha R $. Denote the batch mean and std of $R'$ by $\mu'$ and $\sigma'$. Then the standardized advantage satisfies
\begin{equation}
B(\tau)=\frac{R'(\tau)-\mu'}{\sigma'} = \frac{R(\tau)-\mu}{\sigma} = A(\tau).
\end{equation}
\end{lemma}

\begin{proof}\renewcommand{\qedsymbol}{}
Under $R'(\tau)=\alpha R(\tau)$, we have
$
\mu'=\mathbb{E}[R']=\alpha\mu,
\sigma'=\sqrt{\mathbb{E}[(R'-\mu')^2]}=\alpha\sigma.
$
Hence:
\begin{equation}
B(\tau)=\frac{\alpha R(\tau)-(\alpha\mu)}{\alpha\sigma}
=\frac{\alpha(R(\tau)-\mu)}{\alpha\sigma}=A(\tau).
\end{equation}
\end{proof}
Next, we will demonstrate the scaling invariance of the SPAI method:
\begin{proof}\renewcommand{\qedsymbol}{}
Due to $R'(\tau) = \alpha R(\tau)$, ${r'_{i,t}} = \alpha{r_{i,t}}$,
\begin{equation}
\label{subgs2}
z'_{i,t} = \frac{\alpha r_{i,t}}{\alpha\sqrt{\sum_{i=1}^{G} r_{i,t}^2}} = z_{i,t}.
\end{equation}
From this, it can be seen that the normalization remains unchanged before and after, so for the subsequent steps, we obtain ${W'_{i}}={W_{i}}$.
\end{proof}

\begin{theorem}Scale invariance of injected advantage.
Under any affine transform $R'(\tau)=\alpha R(\tau)$ with $\alpha>0$,
the $A(\tau)$ and the SPAI score $W(\tau)$ are invariant,
hence the injected advantage
\begin{equation}
A'(\tau) =A (\tau)+ A (\tau)\cdot W_i.
\end{equation}
is invariant as well.
\end{theorem}
\paragraph{Remark.}
In practice it is common to compute a numerically-stabilized score
$
F(\tau)=\frac{D^-(\tau)}{D^+(\tau)+D^-(\tau)+\epsilon},
$
with a small $\epsilon>0$ to avoid division-by-zero. In practice, this occurs extremely rarely, so the impact can be considered negligible.

\subsection{Discriminability via Distributional Sparsity}
\label{prop:discriminability}
Let $\tau_a$ and $\tau_b$ be two trajectories yielding identical scalar rewards $R$, ending at time steps $t_a$ and $t_b$ respectively ($t_a \neq t_b$).
Under the SPAI mechanism, their injected advantages are distinct ($A'(\tau_a) \neq A'(\tau_b)$) if and only if the \textbf{reward distribution density} at these time steps differs.
Specifically, if the aggregate reward norms at the two time steps are unequal:
\begin{equation}
\sum_{j=1}^{G} r_{j,t_a}^2 \neq \sum_{j=1}^{G} r_{j,t_b}^2,
\end{equation}
then the trajectories are distinguishable in the latent structural space.

% Let $\mathcal{S}_t = \sqrt{\sum_{j=1}^{G} r_{j,t}^2}$ denote the normalization denominator (sparsity factor) at time step $t$.
\begin{proof}\renewcommand{\qedsymbol}{}
The normalized reward value for a trajectory $i$ ending at $t$ with raw reward $R$ is:
\begin{equation}
z_{i,t} = \frac{R}{\|\mathbf{r}_{t}\|_2 + \mathbbm{1}_{\{\|\mathbf{r}_{t}\|_2 = 0\}}}.
\end{equation}

Consider two trajectories $\tau_a$ and $\tau_b$ ending at $t_a$ and $t_b$ with $R_a=R_b=R$.
If the distribution of valid solutions in the current batch is anisotropic, then the normalized peak values differ:
\begin{equation}
z_{a, t_a} \neq z_{b, t_b}.
\end{equation}
Consequently, the virtual ideal solution components at these steps differ: $z_{t_a}^+ \neq z_{t_b}^+$.
Computing the Euclidean distance to the positive ideal solution $\mathbf{Z}^+$:
\begin{align}
D_a^+ &= \sqrt{(z_{a,t_a} - z_{t_a}^+)^2 + \sum_{t \neq t_a} (0 - z_t^+)^2} \\
D_b^+ &= \sqrt{(z_{b,t_b} - z_{t_b}^+)^2 + \sum_{t \neq t_b} (0 - z_t^+)^2}
\end{align}
Since $z_{a,t_a}$ is the contribution to the ideal vector at $t_a$ (assuming $R$ is the max reward), the distance terms depend on the magnitudes of other components in $\mathbf{Z}^+$. Even in the simplified case where $R$ is constant, the normalization scales the entire vector dimension $t_a$ differently from $t_b$.
Thus, $D_a^+ \neq D_b^+ \implies F_a \neq F_b \implies W_a \neq W_b$.

\textbf{Remark:} In practice, the distribution of reasoning steps (CoT length) typically follows a gaussian-like or long-tail distribution rather than a uniform distribution. Therefore, $\mathcal{S}_{t_a} \neq \mathcal{S}_{t_b}$ holds almost surely, enabling SPAI to adaptively weigh trajectories based on the rarity of their solution path length.
\end{proof}

% \begin{figure}[t]
%     \centering
%     \includegraphics[width=0.48\textwidth]{sec/fig1cv4_cropped.pdf} % 替换为你的图片
%     \caption{The Verification Reward Tensor.}
%     \label{fig:fulu1cv}
% \end{figure}

\subsection{Visualization and Validation}

\begin{figure}[t]
    \centering
    \includegraphics[width=0.48\textwidth]{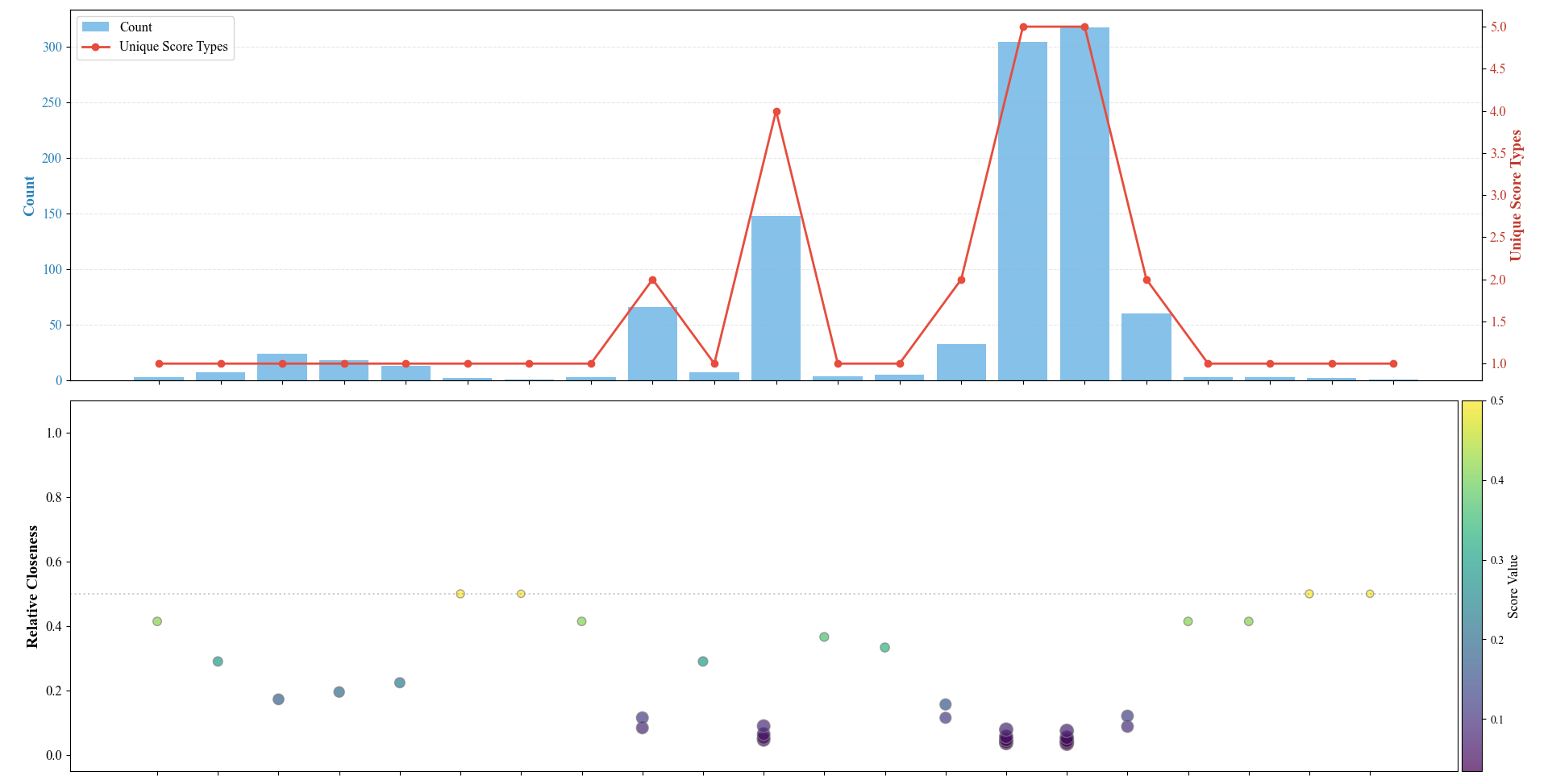} % 替换为你的图片
    \caption{SPAI Visual Analysis. The upper figure shows the analysis of isomerism phenomenon, and the lower figure shows the scores of different rewards from a global perspective.}
    \label{fig:fulua}
\end{figure}

% 为通过实验验证，我们对中间阶段的的SPAI得分进行了统计分析，图的上方展示了同分不同奖励轨迹的现象，可以发现我们的方法在总体奖励相同的情况下，能够实现进一步的细粒度区分。图的下方展示了不同奖励在全局视野的分别情况，我们的方法能够对其在全局视野中对优势信号进行进一步有效区分，促进多样化探索。

As shown in \cref{fig:fulua}, to provide experimental validation, we conducted a statistical analysis of SPAI scores at an intermediate training stage. The upper panel of the figure illustrates the phenomenon of trajectories yielding identical total rewards; it is evident that our method achieves fine-grained differentiation even when total rewards are the same. The lower panel displays the distribution of various rewards from a global perspective, demonstrating that our approach effectively discriminates advantage signals globally to facilitate diversified exploration.

\section{Context Extraction Agent}
\label{sub2}
For the constructed refining agent, there is almost no training delay because it directly reuses the updatable policy model and embeds it into the entire training loop. As shown in \cref{fig:7}, this method can effectively solve the current tool redundancy problem and achieve a compression efficiency of over three-quarters.

\begin{figure}[ht]
    \centering
    % === 外层大框 ===
    \begin{tcolorbox}[
        enhanced,
        colback=gray!3,          % 极淡的灰色背景
        colframe=black!75,       % 深灰色边框
        colbacktitle=gray!15,    % 标题背景色（淡雅）
        coltitle=black,          % 标题文字颜色
        boxrule=0.8pt,
        title=Input Query: New START treaty and number of deployed long-range nuclear warheads,
        fonttitle=\small\bfseries,
        halign title=left
    ]
        
        % === 1. 原始结果子框 (Raw) ===
        % 设计思路：使用稍浊的灰色，代表“未处理/冗余”的状态
        \begin{tcolorbox}[
            colback=gray!10,           % 背景稍深一点的灰色
            colframe=gray!50,          % 中灰色边框
            colbacktitle=gray!30,      % 标题栏颜色
            coltitle=black,
            boxrule=0.5pt,
            title=Raw Search Results (Redundant),
            fonttitle=\scriptsize\bfseries,
            halign title=left,
            left=4pt, right=4pt, top=4pt, bottom=4pt
        ]
            % 正文
            {\scriptsize \texttt{'result': [[{'document': {'id': '3255926', 'contents': "The 'New START' treaty is an agreement by both the US and Russian governments to limit the deploying of nuclear ballistic missiles... [Truncated 200+ words] ...An advisory opinion on this issue was originally requested by the World Health Organization (WHO) on 3 September 1993..."}}]]}} 
            
            \vspace{2mm}
            % 统计信息：改为深灰色
            {\scriptsize \textcolor{gray!60!black}{\textbf{Word Count: $\approx$ 244}}}
        \end{tcolorbox}
        
        % === 中间的箭头 ===
        \vspace{1mm}
        \hspace{5pt} $\downarrow$ \textbf{\small Agent Refinement \& Compression}
        \vspace{1mm}

        % === 2. 提炼结果子框 (Refined) ===
        % 设计思路：使用纯白背景 + 深色强调，代表“清晰/高价值”的状态
        \begin{tcolorbox}[
            colback=white,             % 纯白背景，突出清晰度
            colframe=black!80,         % 近乎黑色的边框，突出重点
            colbacktitle=black!70,     % 深色标题栏，高对比度
            coltitle=white,            % 白色标题文字
            boxrule=0.8pt,             % 稍微加粗一点边框
            title=Refined Agent Output (Concise),
            fonttitle=\scriptsize\bfseries,
            halign title=left,
            left=4pt, right=4pt, top=4pt, bottom=4pt
        ]
            % 正文
            {\scriptsize The ... effective until 2026, limits each party to no more than 1,550 strategic warheads and 700 launchers deployed, ensuring limited nuclear deployment.}
            
            \vspace{2mm}
            % 统计信息：改为黑色加粗
            {\scriptsize \textbf{Word Count: 45 (Reduction: 82\%)}}
        \end{tcolorbox}
        
    \end{tcolorbox}
    \caption{Comparison of information density. The agent compresses raw search results by approximately 82\% while retaining all critical facts, demonstrating its effectiveness in information distillation.}
    \label{fig:7}
\end{figure}

\section{BridgeVQA Information Statistics}
\label{sub3}
As shown in \cref{fig:8}, we visualized the constructed complex dataset and found that its sources are quite diverse. The most frequent category related to sports has only 138 entries, with 875 entries having a frequency greater than 5, and the remaining 2727 entries.

\begin{figure}[htbp]
    \centering
    \begin{subfigure}[b]{0.23\textwidth}  % 稍微调大一点
        \centering
        \includegraphics[width=\linewidth]{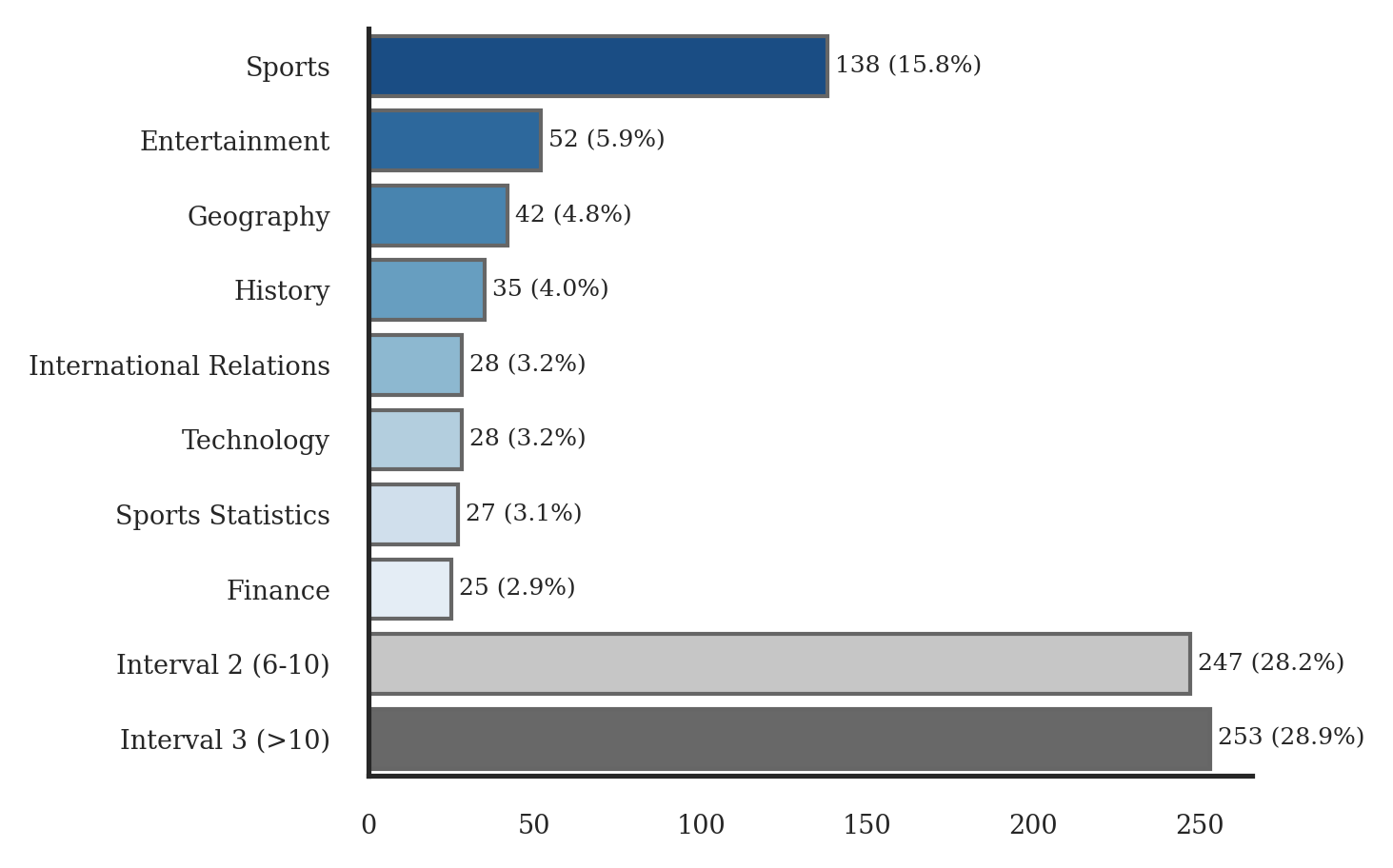}
        \label{fig:a}
        \vspace{-10pt} 
        \caption{}
    \end{subfigure}
    \hfill
    \begin{subfigure}[b]{0.23\textwidth}  % 稍微调大一点
        \centering
        \includegraphics[width=\linewidth]{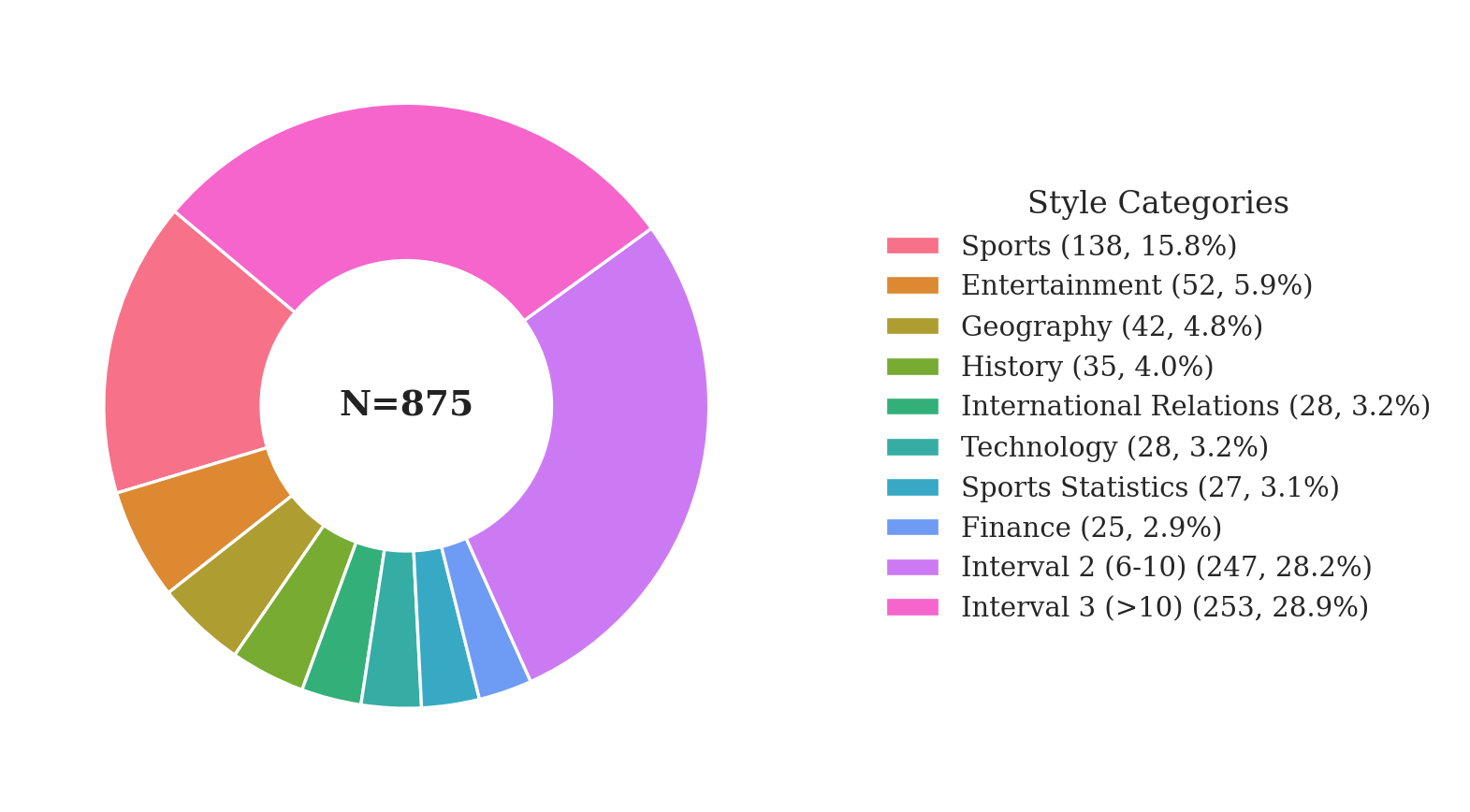}
        \label{fig:b}
        \vspace{-10pt} 
        \caption{}
    \end{subfigure}
    \caption{Statistical information of BridgeVQA. It shows the diversity sources of the dataset, with 875 sources appearing more than 5 times, and the remaining 2727 sources.}
    \label{fig:8}
\end{figure}

\section{Difficulty test of BridgeVQA}
\label{sub4}
We conducted a test on the difficulty level of BridgeVQA. Specifically, we allocated another part as the test set. As shown in \cref{tab:6}, compared to existing datasets, it can be found that BridgeVQA is relatively difficult, and our dataset can definitely guarantee correct answers because it comes from real web pages.

\begin{table}[h]
\centering
\caption{Performance Comparison on FVQA-test and BridgeVQA Datasets}
\label{tab:6}
% 1. 设置较大的列间距以填满单栏（3列建议设置在 12mm-15mm 之间）
\setlength{\tabcolsep}{3.7mm} 
{
    \begin{tabular}{l|c|c}
    \hline
    \textbf{Model} & \textbf{FVQA-test} & \textbf{BridgeVQA} \\ \hline
    \rowcolor{tablegray}
    % 注意：\multicolumn 内部也受 tabcolsep 影响，排版会保持同步
    \multicolumn{3}{c}{\textit{Direct}} \\ \hline
    Qwen2.5-VL-7B & 20.9 & 9.6 \\
    Qwen2.5-VL-32B & 24.7 & 13.4 \\ \hline
    \rowcolor{tablegray}
    \multicolumn{3}{c}{\textit{Search Agent}} \\ \hline
    Qwen2.5-VL-7B & 34.2 & 23.9 \\
    Qwen2.5-VL-32B & 51.3 & 34.2 \\ 
    MMSearch-R1 & 58.4 & 35.3 \\ \hline
    \textbf{Ours} & \textbf{66.8} & \textbf{44.6} \\ \hline
    \end{tabular}
}
\end{table}

\begin{figure}[t]
    \centering
    \includegraphics[width=0.48\textwidth]{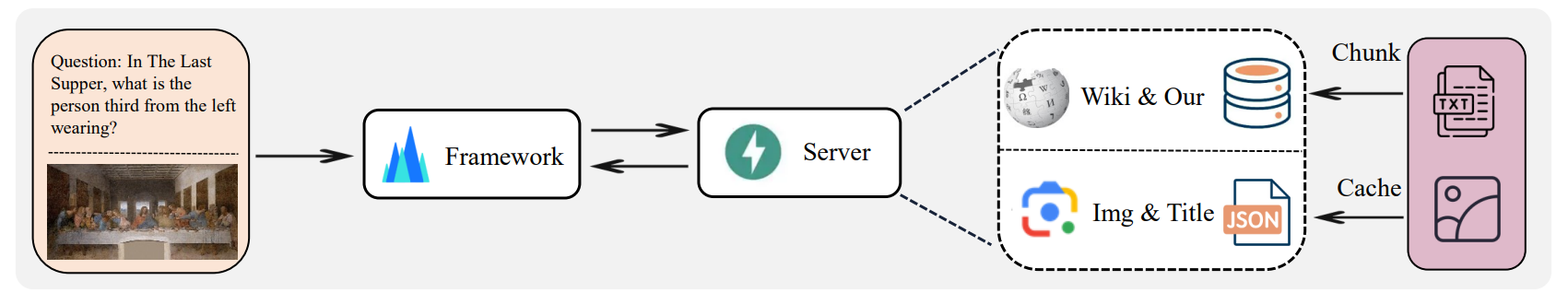} % 替换为你的图片
    \caption{The proposed updatable tool framework. Our tool services include text search and image search. The text search knowledge base is composed of sliced web page information and wikis, while the image search is powered by Google Search and stored in a cached format.}
    \label{fig:9}
\end{figure}

\begin{table}[h!]
\centering
\caption{Performance and cost-efficiency trade-off, measured after one epoch of training.}
\label{tab:7}
% 1. 移除 \resizebox，使用固定列间距确保字体大小为标准大小
% 4列的情况下，尝试 8mm 到 9mm 可以达到铺满效果
\setlength{\tabcolsep}{1.8mm} 
{
    \begin{tabular}{l|c|c|c}
    \toprule
    \textbf{Tools} & \textbf{Acc} & \textbf{Cost} & \textbf{Call Time (s)} \\
    \midrule
    \texttt{Google} & 52.7\% & \$0.50 / 1k calls & 6017 \\
    \rowcolor{gray!10} % 建议给 Ours 这一行加个底色，与前表风格统一
    \texttt{Ours} & \textbf{50.0\%} & \textbf{Virtually \$0.00} & \textbf{457} \\
    \bottomrule
    \end{tabular}
}
\end{table}
% \begin{table}[h!]
% \centering
% \resizebox{\columnwidth}{!}{%
% \begin{tabular}{l|c|c|c}
% \toprule
% \textbf{Tools} & \textbf{Acc} & \textbf{Cost} & \textbf{Call Time (s)} \\
% \midrule
% \texttt{Google Serper} & 52.7\% & \$0.50 / 1k calls & 6017 \\
% \texttt{Ours} & 50.0\% & Virtually \$0.00 & 457 \\

% \bottomrule

% \end{tabular}%
% }
% \caption{Performance and cost-efficiency trade-off, measured after one epoch of training. Accuracy is averaged across all 5 test benchmarks.}
% \label{tab:local&online}
% \end{table}

% \begin{table}[h]
% \centering
% \caption{Performance Comparison on FVQA-test and BridgeVQA Datasets}
% \label{tab:performance_comparison}
% \begin{tabular}{lcc}
% \hline
% \textbf{Model} & \textbf{FVQA-test} & \textbf{BridgeVQA} \\ \hline
% \rowcolor{tablegray}
% \multicolumn{3}{c}{\textit{Direct}} \\ \hline
% Qwen2.5-VL-7B & 20.9 & 9.6 \\
% Qwen2.5-VL-32B & 24.7 & 13.4 \\ \hline
% \rowcolor{tablegray}
% \multicolumn{3}{c}{\textit{ReAct Agent}} \\ \hline
% Qwen2.5-VL-7B & 34.2 & 23.9 \\
% Qwen2.5-VL-32B & 51.3 & 34.2 \\ \hline
% Ours & \textbf{66.8} & \textbf{44.6} \\ \hline
% \end{tabular}
% \end{table}

\begin{table}[ht]
\centering
\caption{Comparison under Identical Evaluation Criteria.}
\label{tab:performance_comparison}
\begin{tabular}{cccc}
\toprule
\textbf{EM} & \textbf{LLM-as-Judge} & \textbf{MMSearch-R1} & \textbf{Ours} \\ 
\midrule
$\checkmark$ &              & 42.5\% & 48.2\% \\ \hline
             & $\checkmark$ & 44.3\% & 50.0\% \\
\bottomrule
\end{tabular}
\end{table}

\section{The Verification of Training Tool Service}
\label{sub5}
To verify the advantages of local upgradable tool services, we validated their scalability and low-cost reliable training features. Specifically, from \cref{fig:9}, it can be observed that our tool is decoupled from the training environment. By slicing the textual information of the constructed data and pre retrieving and caching web page images, a tool service is ultimately formed. This process can be used to construct any text and image information, achieving a closed loop from data to tools. From \cref{tab:7}, it can be seen that our training tool service has good training efficiency and is faster in search speed than online services. Furthermore, we have achieved very low cost consumption and low precision loss.

\section{Specific Information}
\label{sub6}
To ensure a fair and rigorous head-to-head comparison, our evaluation strictly adheres to the standardized data partitions used in previous state-of-the-art works. Specifically, for Infoseek, we evaluate on the same curated subset of 2,000 samples. For SimpleVQA, we restrict our analysis to the 1,013 English-only QA pairs to maintain consistency. Our evaluation on MMSearch is conducted on the selection of 171 image-based samples, and for LiveVQA, we utilize the 2,384-sample version currently available to align with the latest reported benchmarks.

As illustrated in \cref{tab:performance_comparison}, we present a fair comparison under unified evaluation protocols. It is evident that our approach consistently achieves superior performance across both Exact Match (EM) and LLM-as-Judge metrics.

\section{Prompt}
\label{sub7}
We have listed system prompt, summarized prompt and judge prompt.

\begin{tcolorbox}[
    enhanced,
    colback=promptbody,
    colframe=promptheader,
    coltitle=white,
    title=\textbf{System Prompt for Agent},
    fonttitle=\bfseries\large,
    arc=3mm,
    boxrule=0.5mm,
    left=2mm, right=2mm, top=2mm, bottom=2mm,
    standard jigsaw,
    opacityback=1,
    breakable
]
    % --- Part 1: Role & Tools ---
    You are a helpful assistant. You can call functions to assist with the user query.
    
    \textbf{Important:} You must call only one function at a time. After each function call, wait for the execution result before making the next function call if needed.
    
    \vspace{0.5em}
    \textbf{\# Tools}
    
    You are provided with function signatures within \texttt{<tools></tools>} XML tags:
    
    \vspace{0.5em}
    \begin{small}
    \end{small}
    
    \vspace{1em}
    \hrule
    \vspace{1em}

    % --- Part 2: Workflow and Output Format ---
    \textbf{Workflow and Output Format}
    
    You must follow these steps in order. Every conversation turn should start from Step 1.
    
    \vspace{0.5em}
    \textbf{Step 1: Think}
    \begin{itemize}[leftmargin=*, nosep]
        \item \textbf{This is the starting point for every turn.}
        \item Carefully analyze the user's query, break down the problem, and formulate a plan.
        \item Your entire reasoning process must be enclosed in \texttt{<think>...</think>} tags.
    \end{itemize}
    
    \vspace{0.5em}
    \textbf{Step 2: Act (Tool Call)}
    \begin{itemize}[leftmargin=*, nosep]
        \item If your plan requires information you don't have, call \textbf{one single tool}.
        \item The tool call must be enclosed in \texttt{<tool\_call>...</tool\_call>} tags.
        \item If you can answer without tools, \textbf{skip this step} and move directly to Step 4.
    \end{itemize}
    
    \vspace{0.5em}
    \textbf{Step 3: Observe and Think Again}
    \begin{itemize}[leftmargin=*, nosep]
        \item This is a \textbf{critical step}. Receive the tool's observation and analyze it.
        \item You \textbf{MUST} start a new thought process in \texttt{<think>...</think>} tags to analyze this output.
        \item \textbf{Decide:} If information is still insufficient, return to \textbf{Step 2}. If sufficient, proceed to \textbf{Step 4}.
    \end{itemize}
    
    \vspace{0.5em}
    \textbf{Step 4: Summarize}
    \begin{itemize}[leftmargin=*, nosep]
        \item \textbf{Execute only when the decision in Step 3 is to provide the final answer.}
        \item Summarize gathered info in \texttt{<summary>...</summary>} tags:
        \begin{enumerate}[nosep, label=\arabic*.]
            \item \textbf{Extract key points}: Identify and condense relevant information.
            \item \textbf{Eliminate redundancy}: Remove repetitive details for clarity.
            \item \textbf{Verify consistency}: Ensure alignment with query and avoid hallucinations.
        \end{enumerate}
    \end{itemize}
    
    \vspace{0.5em}
    \textbf{Step 5: Answer}
    \begin{itemize}[leftmargin=*, nosep]
        \item Formulate the final answer based on the Step 4 summary.
        \item \textbf{Constraint:} The answer must \textbf{not exceed 30 words}.
        \item Enclose your final answer in \texttt{<answer>...</answer>} tags.
    \end{itemize}

\end{tcolorbox}

\begin{tcolorbox}[
    enhanced,
    colback=promptbody,
    colframe=promptheader,
    coltitle=white,
    title=\textbf{Summary Prompt for Extract Agent},
    fonttitle=\bfseries\large,
    arc=3mm,
    boxrule=0.5mm,
    left=2mm, right=2mm, top=2mm, bottom=2mm,
    standard jigsaw,
    opacityback=1,
    breakable
]

**Role**: You are a professional and concise information synthesizer.

**Task**: Synthesize a **key findings summary** relevant to the "Question" based on the "Evidence/Reasoning" (i.e., the collection of search results).

**Constraints**: The summary must be **strictly limited to under 50 words**. Do not output any titles, roles, extra separators, or explanatory text. Begin the summary directly.

**Input**:
- Question: ${search query}$
- Evidence/Reasoning: ${message}$

**Output Format**: Output only a single, concise summary text, not exceeding 50 words.

\end{tcolorbox}

\begin{tcolorbox}[
    enhanced,
    colback=promptbody,
    colframe=promptheader,
    coltitle=white,
    title=\textbf{Judge Prompt for Answer Evaluation},
    fonttitle=\bfseries\large,
    arc=3mm,
    boxrule=0.5mm,
    left=2mm, right=2mm, top=2mm, bottom=2mm,
    standard jigsaw,
    opacityback=1,
    breakable
]
    % --- Role Definition ---
    Your job is to look at a question, a gold target, and a predicted answer, and then assign a grade of either \texttt{["CORRECT", "INCORRECT", "NOT\_ATTEMPTED"]}.

    \vspace{0.8em}
    % --- Section 1: CORRECT ---
    \textbf{1. CORRECT Predicted Answers}
    \begin{itemize}[leftmargin=*, nosep]
        \item \textbf{Criteria:} Answers must fully contain the gold target information without contradicting it. Semantic meaning matters; capitalization, punctuation, and grammar do not. Hedging is allowed if the target is included.
        \item \textbf{Examples:}
        \begin{small}
        \begin{quote}
            Q: Obama's children? | Gold: Malia and Sasha \\
            A1: sasha and malia obama \\
            A2: Malia Ann and Natasha Marian, commonly referred to as Malia and Sasha...
        \end{quote}
        \end{small}
    \end{itemize}

    \vspace{0.5em}
    % --- Section 2: INCORRECT ---
    \textbf{2. INCORRECT Predicted Answers}
    \begin{itemize}[leftmargin=*, nosep]
        \item \textbf{Criteria:} Any factual statement contradicts the gold target. Hedged incorrect statements (e.g., "it is possible that...") are still incorrect.
        \item \textbf{Examples:}
        \begin{small}
        \begin{quote}
            Gold: Malia and Sasha \\
            A1: Malia, Sasha, and Susan. (Extra info contradicts facts) \\
            A2: Obama has three children. (Factual error)
        \end{quote}
        \end{small}
    \end{itemize}

    \vspace{0.5em}
    % --- Section 3: NOT_ATTEMPTED ---
    \textbf{3. NOT\_ATTEMPTED Predicted Answers}
    \begin{itemize}[leftmargin=*, nosep]
        \item \textbf{Criteria:} The important information is missing, but no statements contradict the gold target.
        \item \textbf{Examples:}
        \begin{small}
        \begin{quote}
            A1: I don't know. \\
            A2: I know one is Malia, but I'm not sure about the other.
        \end{quote}
        \end{small}
    \end{itemize}

    \vspace{1em}
    \hrule
    \vspace{1em}

    % --- Grading Instance ---
    \textbf{Grading Task} \\
    Simply reply with the letter corresponding to the grade. Do not apologize or add extra text.
    
    \vspace{0.5em}
    \texttt{Question: \{question\}} \\
    \texttt{Gold target: \{correct\_answer\}} \\
    \texttt{Predicted answer: \{response\}}
    
    \vspace{0.5em}
    \textbf{Options:}
    \begin{itemize}[leftmargin=*, nosep]
        \item \textbf{A:} CORRECT
        \item \textbf{B:} INCORRECT
        \item \textbf{C:} NOT\_ATTEMPTED
    \end{itemize}

    \vspace{0.5em}
    \textbf{Output:} Return only "A", "B", or "C".
\end{tcolorbox}

\section{Case Study}

% 定义颜色，让盒子看起来更专业（模仿 OpenAI/DeepSeek 风格）
\definecolor{promptheader}{RGB}{70, 130, 180} % SteelBlue
\definecolor{promptbody}{RGB}{245, 248, 250}   % AliceBlue
\definecolor{thinkcolor}{RGB}{105, 105, 105}   % DimGray for thinking process
\definecolor{answercolor}{RGB}{0, 100, 0}     % DarkGreen for final answer
\begin{tcolorbox}[
    enhanced,
    colback=promptbody,
    colframe=promptheader,
    coltitle=white,
    title=\textbf{An Example of BridgeVQA},
    fonttitle=\bfseries\large,
    arc=3mm,
    boxrule=0.5mm,
    left=3mm, right=3mm, top=3mm, bottom=3mm, % 稍微增加一点内边距
    standard jigsaw,
    opacityback=1,
    breakable
]
    % 1. Question 部分
    \textbf{\faQuestionCircle\ Question:} What is the name of the horseback riding event that the woman who celebrated her 83rd birthday with the man in the image attended, while both wore matching helmets, black jackets, and white pants?
    
    % 2. Image 部分 (居中显示，调整这一行的 vspace 控制间距)
    \begin{center}
        \includegraphics[width=0.85\linewidth, height=4cm, keepaspectratio]{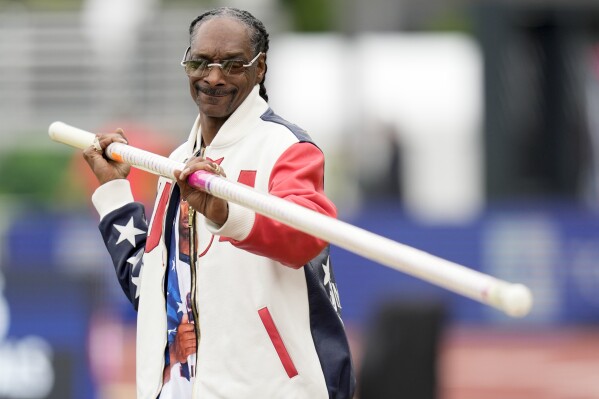}
    \end{center}
    
    % 3. Response 部分
    \textbf{COT:}
    
    % Thinking Process (用灰色显示，模拟思维链)
    \vspace{0.5em}
    {\color{thinkcolor}
    % \texttt{<think>}
    1. [Visual Step] Identify the image shows 'the man in the image'. 2. [Text Search 1] Search for 'the man in the image' + '83rd birthday'. Found linkage to Martha Stewart celebrating her 83rd birthday with Snoop Dogg in paragraph [X]. 3. [Text Search 2] Now search for 'Martha Stewart horseback riding event' to find the specific detail requested. Found 'attended a equestrian team dressage event' in paragraph [Y]. 4. [Conclusion] The answer is Dressage Team Grand Prix Final.}
    % \texttt{</think>}}
    
    % Final Answer (用加粗或不同颜色强调)
    \vspace{0.5em}
    {\color{answercolor}
    \textbf{\texttt{<answer>}
    Dressage Team Grand Prix Final.
    \texttt{</answer>}}}
\end{tcolorbox}

\definecolor{promptheader}{RGB}{70, 130, 180} 
\definecolor{promptbody}{RGB}{245, 248, 250}   
\definecolor{thinkcolor}{RGB}{105, 105, 105}   
\definecolor{toolcolor}{RGB}{178, 34, 34}      % FireBrick for tool calls
\definecolor{obsframe}{RGB}{220, 220, 220}     % Light gray for observation box
\definecolor{answercolor}{RGB}{0, 100, 0}

\begin{tcolorbox}[
    enhanced,
    colback=promptbody,
    colframe=promptheader,
    coltitle=white,
    title={\textbf{A Case of Fvqa}},
    fonttitle=\bfseries\large,
    arc=3mm,
    boxrule=0.5mm,
    left=3mm, right=3mm, top=3mm, bottom=3mm,
    standard jigsaw,
    opacityback=1,
    breakable
]

    % --- 1. User Input ---
    \textbf{\faQuestionCircle\ Question:} Where was the latest festival for impressive traditional Hina dolls hosted?
    
    \begin{center}
        \includegraphics[height=3.5cm, keepaspectratio]{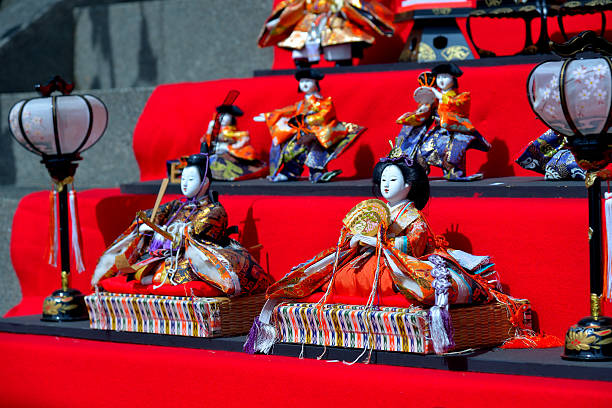}
    \end{center}

    \noindent\rule{\linewidth}{0.4pt} % 分隔线
    
    % --- Step 1 ---
    \textbf{Step 1}
    
    {\color{thinkcolor}\small \texttt{<think>} The images and titles provided are related to vintage Hina dolls, but they do not directly indicate the location of the latest festival for these dolls. We need to focus on finding information about recent festivals featuring traditional Hina dolls. \texttt{</think>}}
    
    \vspace{0.3em}
    {\color{toolcolor}\small  \textbf{Tool Call:} \texttt{web\_image\_to\_image\_search}}
    
    \begin{tcolorbox}[colback=white, colframe=obsframe, title={\footnotesize \faEye\ Tool Response}, sharp corners, boxrule=0.3mm, left=1mm, right=1mm, top=1mm, bottom=1mm]
        \footnotesize
        \textbf{Results:} 1. Vintage Japanese Hina Doll... 2. Traditional Folk Craft... 3. Girls' Festival dolls...\\
        \textbf{Retrieved Images:}
        \begin{center}
            \begin{minipage}[c]{0.32\linewidth}
                \centering
                \includegraphics[height=2.2cm, max width=\linewidth, keepaspectratio, valign=m]{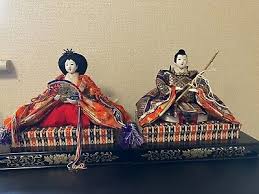}
            \end{minipage}%
            \hfill
            \begin{minipage}[c]{0.32\linewidth}
                \centering
                \includegraphics[height=2.2cm, max width=\linewidth, keepaspectratio, valign=m]{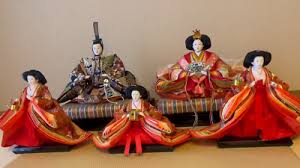}
            \end{minipage}%
            \hfill
            \begin{minipage}[c]{0.32\linewidth}
                \centering
                \includegraphics[height=2.2cm, max width=\linewidth, keepaspectratio, valign=m]{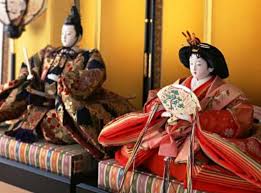}
            \end{minipage}
        \end{center}
    \end{tcolorbox}

    % --- Step 2 ---
    \textbf{Step 2}
    
    {\color{thinkcolor}\small  \texttt{<think>} The information provided by the Wikipedia search indicates that the latest festival for impressive traditional Hina dolls was held in Tokyo, at specific venues like Inatori Onsen and Hotel Gajoen Tokyo. \texttt{</think>}}
    
    \vspace{0.3em}
    {\color{toolcolor}\small  \textbf{Tool Call:} \texttt{web\_text\_search(query="Where was the latest festival for impressive traditional Hina dolls hosted?")}}
    
    \begin{tcolorbox}[colback=white, colframe=obsframe, title={\footnotesize \faList\ Tool Response}, sharp corners, boxrule=0.3mm, left=1mm, right=1mm, top=1mm, bottom=1mm]
        \footnotesize
        \textbf{Observation:} The latest festival for impressive traditional Hina dolls was hosted in \textbf{Tokyo}, specifically at \textbf{Inatori Onsen} and at \textbf{Hotel Gajoen Tokyo}.
    \end{tcolorbox}

    % --- Step 3 ---
    \textbf{Step 3}
    
    {\color{thinkcolor}\small \texttt{<think>} The response confirms that the most recent Hina Matsuri was held on March 3, 2025. It is traditionally celebrated annually on this date in Japan. \texttt{</think>}}
    
    \vspace{0.3em}
    {\color{toolcolor}\small  \textbf{Tool Call:} \texttt{web\_text\_search(query="When did the most recent Hina Matsuri take place?")}}
    
    \begin{tcolorbox}[colback=white, colframe=obsframe, title={\footnotesize \faList\ Tool Response}, sharp corners, boxrule=0.3mm, left=1mm, right=1mm, top=1mm, bottom=1mm]
        \footnotesize
        \textbf{Observation:} The most recent Hina Matsuri (Girl's Day) took place on \textbf{March 3, 2025}.
    \end{tcolorbox}

    % --- Summary Section ---
    \vspace{0.5em}
    \noindent\rule{\linewidth}{0.4pt}
    \vspace{0.5em}
    {\color{thinkcolor}\small \texttt{<summary>} The latest festival for impressive traditional Hina dolls... The most recent Hina Matsuri took place on March 3, 2025, in Japan. \texttt{</summary>}}

    \vspace{0.8em}
    % --- Final Answer ---
    {\color{answercolor}
    \textbf{\faCheckCircle\ \texttt{<answer>}
    The latest festival for impressive traditional Hina dolls was hosted in Tokyo, specifically at Inatori Onsen and Hotel Gajoen Tokyo, during the most recent Hina Matsuri on March 3, 2025. This event is part of the annual celebration of Girl's Day in Japan.
    \texttt{</answer>}}}

\end{tcolorbox}

  % \caption{The proposed updatable tool framework. Our tool services include text search and image search. The text search knowledge base is composed of sliced web page information and wikis, while the image search is powered by Google Search and stored in a cached format.}
% \subsection{Implications}

% \paragraph{(1) Theoretical safety.}
% Reward-scale invariance guarantees that replacing raw reward with TOPSIS score does not change the optimal policy.

% \paragraph{(2) Exploration enhancement.}
% Equal raw rewards no longer collapse to equal advantages; geometric distinctions increase diversity pressure.

% \paragraph{(3) Faster learning.}
% Higher SNR provides more stable gradients and reduces sample inefficiency.

% \paragraph{(4) Flexibility.}
% The framework allows flexible design of multidimensional evaluation signals while preserving policy optimality.


% You can have as much text here as you want. The main body must be at most $8$
% pages long. For the final version, one more page can be added. If you want, you
% can use an appendix like this one.

% The $\mathtt{\backslash onecolumn}$ command above can be kept in place if you
% prefer a one-column appendix, or can be removed if you prefer a two-column
% appendix.  Apart from this possible change, the style (font size, spacing,
% margins, page numbering, etc.) should be kept the same as the main body.
%%%%%%%%%%%%%%%%%%%%%%%%%%%%%%%%%%%%%%%%%%%%%%%%%%%%%%%%%%%%%%%%%%%%%%%%%%%%%%%
%%%%%%%%%%%%%%%%%%%%%%%%%%%%%%%%%%%%%%%%%%%%%%%%%%%%%%%%%%%%%%%%%%%%%%%%%%%%%%%

\end{document}